\documentclass[pdflatex,sn-mathphys-num]{sn-jnl}

\usepackage{graphicx}%
\usepackage{multicol}%
\usepackage{multirow}
\usepackage{amsmath,amssymb,amsfonts}%
\usepackage{amsthm}%
\usepackage{mathrsfs}%
\usepackage[title]{appendix}%
\usepackage{xcolor}%
\usepackage{textcomp}%
\usepackage{manyfoot}%
\usepackage{booktabs}%
\usepackage{algorithm}%
\usepackage{algorithmicx}%
\usepackage{algpseudocode}%
\usepackage{listings}%
\usepackage{subfig}

\usepackage[commandnameprefix=always, defaultcolor=red,final]{changes}
\usepackage{hhline}
\usepackage{enumitem}
\usepackage{cleveref} 
\usepackage[most]{tcolorbox}
\usepackage{makecell}
\usepackage[algo2e]{algorithm2e}
\usepackage{float}

\theoremstyle{thmstyleone}%
\newtheorem{theorem}{Theorem}
\theoremstyle{thmstyletwo}%
\newtheorem{example}{Example}%

\theoremstyle{thmstylethree}%
\newtheorem{definition}{Definition}%

\begin{document}
\newtheorem{apdthm}{Theorem}
\newtheorem{lemma}{Lemma}

\newcommand{\viatraStyle}[1]{\lstinline[style=viatrabig]@#1@}



\newcommand{\dslStyle}[1]{\texttt{\color{emphColor}{#1}}}
\newcommand{\dslVar}[1]{\mathit{#1}}
\newcommand{\dslInstanceOf}[2]{\dslStyle{#1}(#2)}
\newcommand{\dslInReference}[3]{\dslStyle{#1}(#2,#3)}

\newcommand{\term}{}
\newcommand{\termConstantInd}[2]{{#1}_{#2}}
\newcommand{\termConstant}[1]{{#1}}

\newcommand{\termLiteralInd}[3]{{#1}_{#3}(\termConstant{#2})}
\newcommand{\termLiteral}[2]{\termLiteralInd{#1}{#2}{}}

\newcommand{\symbolClass}[2]{\dslStyle{#1}_{#2}}
\newcommand{\termClassInd}[3]{\symbolClass{#1}{#3}(\termConstant{#2})}
\newcommand{\termClass}[2]{\termClassInd{#1}{#2}{}}

\newcommand{\symbolRef}[2]{\dslStyle{#1}_{#2}}
\newcommand{\termReferenceInd}[4]{\symbolRef{#1}{#4}(\termConstant{#2},\termConstant{#3})}
\newcommand{\termReference}[3]{\termReferenceInd{#1}{#2}{#3}{}}

\newcommand{\symbolTCRef}[2]{\dslStyle{#1}_{#2}^+}
\newcommand{\termTCReferenceInd}[4]{\symbolTCRef{#1}{#4}(\termConstant{#2},\termConstant{#3})}
\newcommand{\termTCReference}[3]{\termTCReferenceInd{#1}{#2}{#3}{}}

\newcommand{\symbolHyper}[2]{\dslStyle{#1}_{#2}}
\newcommand{\termHyperedgeInd}[4]{\symbolHyper{#1}{#4}(\termConstant{#2}, \ldots, \termConstant{#3})}
\newcommand{\termHyperedge}[3]{\termHyperedgeInd{#1}{#2}{#3}{}}

\newcommand{\dslExistSymbol}[0]{{\color{emphColor}{\varepsilon}}}
\newcommand{\dslExist}[1]{\dslExistSymbol(#1)}
\newcommand{\termExist}[1]{\dslExist{#1}}

\newcommand{\dslEqualSymbol}[0]{{\color{emphColor}{\sim}}}
\newcommand{\dslEquals}[2]{#1\mathrel{\dslEqualSymbol}#2}
\newcommand{\termEquals}[2]{\dslEquals{#1}{#2}}

\newcommand{\termDistinct}[1]{\mathit{dist}(#1)}


\newcommand{\dslStyleP}[1]{\texttt{#1}}
\newcommand{\dslInPattern}[2]{\dslStyleP{#1}(#2)}
\newcommand{\patternDef}[0]{:=}


\newcommand{\logicModelTwo}[2]{\langle {#1}, {#2} \rangle}
\newcommand{\logicModel}[3]{\langle {#1}, {#2}, {#3} \rangle}
\newcommand{\arity}[1]{\alpha(#1)}

\newcommand{\dslObjectSet}[0]{\mathcal{O}}
\newcommand{\dslObjectSetI}[1]{\dslObjectSet_{#1}}
\newcommand{\dslIdSet}[0]{\mathit{Id}}
\newcommand{\dslTheory}[0]{\mathcal{T}}

\newcommand{\logicStyle}[1]{\texttt{\color{logsemColor}{#1}}}

\newcommand{\kw}[1]{\textup{\texttt{\color{magenta}#1}}}
\newcommand{\true}[0]{\kw{true}}
\newcommand{\false}[0]{\kw{false}}
\newcommand{\unk}[0]{\kw{unknown}}
\newcommand{\err}[0]{\kw{error}}
\newcommand*{\Unk}{\kw{u}}
\newcommand*{\True}{\kw{t}}
\newcommand*{\False}{\kw{f}}
\newcommand*{\Err}{\kw{e}}

\newcommand{\numStyle}[1]{\texttt{\color{numsemColor}{#1}}}
\newcommand{\unbound}{\textup{\numStyle{?}}}

\newcommand{\refinel}[2]{#1 \mathrel{\logsemStyle{\sqsubseteq_L}} #2}
\newcommand{\refinen}[2]{#1 \mathrel{\numsemStyle{\sqsubseteq_N}} #2}
\newcommand{\refine}[2]{#1 \sqsubseteq #2}

\newcommand{\logsemStyle}[1]{{\color{logsemColor}{#1}}}

\newcommand{\semantics}[3]{{\llbracket #1 \rrbracket}^{#2}_{#3}}
\newcommand{\interpretationFun}[1]{\mathcal{I}_{#1}}
\newcommand{\interpretationDef}[4]{\interpretationFun{#1}(#2): {#3} \rightarrow {#4}}
\newcommand{\interpretationApp}[3]{\mathcal{I}_{#1}(#2)(#3)}
\newcommand{\interpretationConst}[2]{\mathcal{I}_{#1}(#2)}

\newcommand{\numsemStyle}[1]{{\color{numsemColor}{#1}}}

\newcommand{\numsems}[3]{{{\color{numsemColor}{\llparenthesis}} #1 {\color{numsemColor}{\rrparenthesis}}}^{#2}_{#3}}
\newcommand{\interpretationNumFun}[1]{\mathcal{V}_{#1}}
\newcommand{\interpretationNumDef}[4]{\interpretationNumFun{#1}(#2): {#3} \rightarrow {#4}}
\newcommand{\interpretationNumApp}[3]{\mathcal{V}_{#1}(#2)(#3)}
\newcommand{\interpretationNumConst}[2]{\mathcal{V}_{#1}(#2)}


\newcommand{\Left}[2]{\mathbf{Left}(#1, #2)}
\newcommand{\Right}[2]{\mathbf{Right}(#1, #2)}
\newcommand{\Behind}[2]{\mathbf{Behind}(#1, #2)}
\newcommand{\Front}[2]{\mathbf{Front}(#1, #2)}
\newcommand{\Above}[2]{\mathbf{Above}(#1, #2)}
\newcommand{\Below}[2]{\mathbf{Below}(#1, #2)}
\newcommand{\Shape}[2]{\mathbf{Shape}(#1, #2)}
\newcommand{\smallSize}[1]{\mathbf{Size}(#1, \textsf{"Small"})}
\newcommand{\largeSize}[1]{\mathbf{Size}(#1, \textsf{"Large"})}
\newcommand{\cubeShape}[1]{\mathbf{Shape}(#1, \textsf{"Cube"})}
\newcommand{\bottom}[1]{\mathbf{Bottom}(#1)}
\newcommand{\Color}[2]{\mathbf{Color}(#1, #2)}
\newcommand{\ColorOf}[1]{\mathbf{Color}(#1)}
\newcommand{\Constraint}[1]{\textbf{\textit{C#1}}}
\newcommand{\literal}[1]{\textsf{"#1"}}
\newcommand{\ConstraintSet}{\mathbf{\Phi}}
\newcommand{\MaterialOf}[1]{\mathbf{Material}(#1)}
\newcommand{\ShapeOf}[1]{\mathbf{Shape}(#1)}
\newcommand{\SizeOf}[1]{\mathbf{Size}(#1)}

\newcommand{\tripleSet}[1]{\mathbf{T}_{#1}}
\newcommand{\sceneSet}[1]{\mathbf{S}_{#1}}

\newcommand{\AttrValSet}{\mathbf{V}_a}
\newcommand{\AttrSet}{\mathbf{A}}
\newcommand{\RelSet}{\mathbf{R}}
\newcommand{\NodeSet}{\mathbf{N}}
\newcommand{\EdgeSet}{\mathbf{E}}
\newcommand{\AttrMap}{\textit{attr}}

\newcommand{\edgeFOL}[3]{\mathbf{#3}(#1, #2)}
\newcommand{\attrFOL}[3]{\mathbf{#3}(#1, #2)}
\newcommand{\attrFunc}[2]{\mathbf{#1}(#2)}

\newcommand{\iso}[2]{#1 \cong #2}
\newcommand{\edge}[3]{\overrightarrow{#1 #2}^{#3}}
\newcommand{\attr}[3]{\overrightarrow{#1 #2}^{#3}}

\newcommand{\argmax}[1]{argmax(#1)}
\newcommand{\perturb}[3]{\mathcal{P}^{#1, #2}_{#3}}
\newcommand{\perturba}[1]{\mathcal{P}_{#1}}
\newcommand{\nodeFeature}[3]{\mathbb{#1}^{(#2, #3)}}

\newcommand{\RQ}[1]{\textbf{RQ#1}}
\newcommand{\ranswer}[2]{%
\medskip
\noindent
\begin{tabularx}{\linewidth}{|X|}\hline
\rquestion{#1}: \emph{#2}\\\hline
\end{tabularx}\medskip}

\newcommand{\absint}[0]{AbsInt}
\newcommand{\localp}[0]{{p_{l}}}
\newcommand{\globalp}[0]{{p_{g}}}
\newcommand{\nodeR}[1]{\mathrm{R}_{\classifier{}}({\graph{}, \perturba{\nodeF{}}}, #1)}
\newcommand{\graphR}[0]{\mathrm{R}_{\classifier{}}(\graph{}, \perturba{\nodeF{}})}
\newcommand{\graphRLower}[0]{\graphR{}_{\le}}
\newcommand{\graphRUp}[0]{\graphR{}_{\ge}}
\newcommand{\graphRobustLower}[2]{\mathrm{R}_{\classifier{}}(\graph{}, \perturb{#1}{#2}{\nodeF{}})_{\le}}
\newcommand{\graphRobustUp}[2]{\mathrm{R}_{\classifier{}}(\graph{}, \perturb{#1}{#2}{\nodeF{}})_{\ge}}

\newcommand{\nodeRC}[1]{\mathrm{R}^{\certifier{}}_{\classifier{}}({\graph{}, \perturba{\nodeF{}}}, #1)}
\newcommand{\nodeRCV}[2]{\mathrm{R}^{#2}_{\classifier{}}({\graph{}, \perturba{\nodeF{}}}, #1)}
\newcommand{\nodeRCP}[3]{\mathrm{R}^{\certifier{}}_{\classifier{}}(\graph{}, \perturb{#2}{#3}{\nodeF{}}, #1)}
\newcommand{\graphRC}[0]{\mathrm{R}^{\certifier{}}_{\classifier{}}(\graph{}, \perturba{\nodeF{}})}
\newcommand{\certifier}[0]{\mathscr{V}}
\newcommand{\nodeVar}[0]{\mathcal{X}}
\newcommand{\operation}[0]{\mathrm{T}}
\newcommand{\op}[1]{\mathrm{#1}}
\newcommand{\poly}[0]{}

\newcommand{\abstractFunc}[0]{\alpha}
\newcommand{\concreteFunc}[0]{\gamma}
\newcommand{\concreteElement}[0]{\set{S}}
\newcommand{\numVariables}[0]{x}
\newcommand{\map}[0]{\set{V}}

\newcommand{\polyLowExp}[1]{Q^{#1}_{\le}}
\newcommand{\polyLowConst}[1]{d^{#1}_{\le}}
\newcommand{\polyUpExp}[1]{Q^{#1}_{\ge}}
\newcommand{\polyUpConst}[1]{d^{#1}_{\ge}}
\newcommand{\diffExp}[1]{D^{#1}_{\lab{}, \lab{}'}}
\newcommand{\diffConst}[1]{h^{#1}_{\lab{}, \lab{}'}}
\newcommand{\diff}[0]{d}

\newcommand{\low}[0]{lo}
\newcommand{\high}[0]{up}
\newcommand{\variable}[0]{\texttt{x}}

\newcommand{\adj}[0]{A}
\newcommand{\nodeF}[0]{X}
\newcommand{\numNodes}[0]{n}
\newcommand{\numFeatures}[0]{m}
\newcommand{\latentFeatures}[0]{H}
\newcommand{\node}[0]{i}
\newcommand{\feature}[0]{j}
\newcommand{\neighbors}[2]{\mathcal{N}^{#1}_{#2}}

\newcommand{\weight}[0]{W}
\newcommand{\bias}[0]{b}
\newcommand{\classifier}[0]{\mathit{cl}}
\newcommand{\lab}[0]{c}
\newcommand{\linear}[1]{\mathrm{Lin}_{#1}}
\newcommand{\gc}[0]{\mathrm{GC}_{\adj}}
\newcommand{\sub}[0]{\mathrm{Sub}}
\newcommand{\relu}[0]{\mathrm{ReLU}}
\newcommand{\numLayers}[0]{z}
\newcommand{\layer}[0]{l}

\newcommand{\identity}[0]{I}

\newcommand{\domain}[1]{\mathbb{#1}}
\newcommand{\binaryDomain}[0]{\domain{B}}
\newcommand{\realDomain}[0]{\domain{R}}
\newcommand{\concreteDomain}[0]{\domain{C}}
\newcommand{\abstractDomain}[0]{\domain{A}}
\newcommand{\intDomain}[0]{\domain{N}}
\newcommand{\expDomain}[0]{\domain{E}}
\newcommand{\mapDomain}[0]{\domain{M}}

\newcommand{\set}[1]{\mathcal{#1}}
\newcommand{\labelSet}[0]{\set{C}}
\newcommand{\operationSet}[0]{\set{T}}

\newcommand{\tuple}[1]{\mathbf{#1}}
\newcommand{\graph}[0]{\tuple{G}}
\newcommand{\absele}[0]{\tuple{a}}
\newcommand{\coef}[0]{\tuple{q}}

\newcommand{\lowBound}[0]{L}
\newcommand{\upBound}[0]{U}

\newcommand{\Dual}[0]{Dual}
\newcommand{\DualF}[0]{Dual-F}
\newcommand{\DualO}[0]{Dual-O}
\newcommand{\Interval}[0]{Interval}
\newcommand{\Ours}[0]{Poly}
\newcommand{\OursM}[0]{Poly-Max}
\newcommand{\OursT}[0]{Poly-TopK}

\newcommand{\bce}[0]{BCE}
\newcommand{\hinge}[0]{RH}

\newcommand{\RMFIX}[0]{\textbf{RM}+\OP}
\newcommand{\PMFIX}[0]{\textbf{PM}+\OP}

\newcommand{\SA}[0]{\textbf{SA}\xspace}
\newcommand{\Con}[0]{\textbf{Con}\xspace}


\newcommand{\fss}[0]{FSS}
\newcommand{\moopt}[0]{MOO}
\newcommand{\mom}[0]{MOM}

\newcommand{\qa}[1]{{\footnotesize\textit{{\textsf{#1}}}}}
\newcommand{\qaTernaryPred}[4]{\qa{#1}(#2, #3, #4)}
\newcommand{\qaBinaryPred}[3]{\qa{#1}(#2, #3)}
\newcommand{\qaUnaryPred}[2]{\qa{#1}(#2)}
\newcommand{\qaPos}[0]{\qa{positional}}
\newcommand{\qaPosPred}[2]{\qaPos(#1, #2)}
\newcommand{\qaLeft}[0]{\qa{left}}
\newcommand{\qaLeftPred}[2]{\qaLeft(#1, #2)}
\newcommand{\qaRight}[0]{\qa{right}}
\newcommand{\qaRightPred}[2]{\qaRight(#1, #2)}
\newcommand{\qaFront}[0]{\qa{ahead}}
\newcommand{\qaFrontPred}[2]{\qaFront(#1, #2)}
\newcommand{\qaBehind}[0]{\qa{behind}}
\newcommand{\qaBehindPred}[2]{\qaBehind(#1, #2)}

\newcommand{\qaDist}[0]{\qa{distance}}
\newcommand{\qaDistPred}[2]{\qaDist(#1, #2)}
\newcommand{\qaClose}[0]{\qa{close}}
\newcommand{\qaClosePred}[2]{\qaClose(#1, #2)}
\newcommand{\qaMed}[0]{\qa{medDist}}
\newcommand{\qaMedPred}[2]{\qaMed(#1, #2)}
\newcommand{\qaFar}[0]{\qa{far}}
\newcommand{\qaFarPred}[2]{\qaFar(#1, #2)}

\newcommand{\qaCanSee}[0]{\qa{canSee}}
\newcommand{\qaCanSeePred}[2]{\qaCanSee(#1, #2)}

\newcommand{\qaNoColl}[0]{\qa{noColl}}
\newcommand{\qaNoCollPred}[2]{\qaNoColl(#1, #2)}

\newcommand{\qaOnRoad}[0]{\qa{onRoad}}
\newcommand{\qaOnRoadPred}[1]{\qaOnRoad(#1, #1)}

\newcommand{\qaCar}[0]{\qa{Car}}
\newcommand{\qaCarPred}[1]{\qaCar(#1)}
\newcommand{\qaPedestrian}[0]{\qa{Pedestrian}}
\newcommand{\qaPedestrianPred}[1]{\qaPedestrian(#1)}



\newcommand{\scenic}{\textsc{Scenic}}
\newcommand{\conType}[1]{\textbf{#1}}

\newcommand{\actorNM}[1]{\vec{\texttt{a}}_#1}
\newcommand{\actor}[1]{$\actorNM{#1}$}

\newcommand{\actorG}[0]{\textcolor{car1}{\actor{G}}}
\newcommand{\actorR}[0]{\textcolor{car2}{\actor{R}}}
\newcommand{\actorBl}[0]{\textcolor{car3}{\actor{B}}}

\newcommand{\actorA}[0]{\textcolor{car2}{\actor{A}}}
\newcommand{\actorB}[0]{\textcolor{gray}{\actor{B}}}

\newcommand{\actorSampleNM}[0]{\actorNM{i}}
\newcommand{\actorSample}[0]{\actor{i}}

\newcommand{\justParam}[1]{\texttt{#1}}
\newcommand{\justParamB}[1]{\textcolor{gray}{\justParam{#1}}}
\newcommand{\param}[2]{\actorNM{#2}.\justParam{#1}}
\newcommand{\paramG}[1]{\textcolor{car1}{\param{#1}{G}}}
\newcommand{\paramR}[1]{\textcolor{car2}{\param{#1}{R}}}
\newcommand{\paramBl}[1]{\textcolor{car3}{\param{#1}{B}}}

\newcommand{\paramA}[1]{\textcolor{car2}{\param{#1}{A}}}
\newcommand{\paramB}[1]{\textcolor{gray}{\param{#1}{B}}}
\newcommand{\paramSample}[1]{\param{#1}{i}}

\newcommand{\actorFunNoMath}[1]{\texttt{o}_#1}
\newcommand{\actorFunNMG}[0]{\textcolor{car1}{\actorFunNoMath{G}}}
\newcommand{\actorFunNMR}[0]{\textcolor{car2}{\actorFunNoMath{R}}}
\newcommand{\actorFunNMBl}[0]{\textcolor{car3}{\actorFunNoMath{B}}}
\newcommand{\actorFunNMA}[0]{\textcolor{car2}{\actorFunNoMath{A}}}
\newcommand{\actorFunNMB}[0]{\textcolor{gray}{\actorFunNoMath{B}}}

\newcommand{\actorFun}[1]{$\actorFunNoMath{#1}$}
\newcommand{\actorFunG}[0]{\textcolor{car1}{\actorFun{G}}}
\newcommand{\actorFunR}[0]{\textcolor{car2}{\actorFun{R}}}
\newcommand{\actorFunBl}[0]{\textcolor{car3}{\actorFun{B}}}

\newcommand{\actorFunA}[0]{\textcolor{car2}{\actorFun{A}}}
\newcommand{\actorFunB}[0]{\textcolor{gray}{\actorFun{B}}}
\newcommand{\actorFunSample}[0]{\actorFun{i}}


\newcommand{\funtolog}[0]{\oldstylenums{f2l}}
\newcommand{\logtofun}[0]{\oldstylenums{l2f}}


\newcommand{\scenTerm}[1]{\emph{#1}}


\newcommand{\experiment}[1]{\underline{\textsc{Exp {#1}}}}
\newcommand{\rquestion}[1]{\textbf{\textsc{RQ}#1}}
\newcommand{\contribution}[1]{\textbf{\textsc{C}#1}}
\newcommand{\contribox}[2]{\boxedVal{\contribution{#1 }}{#2}}

\newcommand{\boxedVal}[2]{%
\noindent
\addvbuffer[0.5\baselineskip]{
\begin{tabularx}{\linewidth}{|X|}
\hline
#1 
#2\\\hline
\end{tabularx}}}

\newcommand{\cmt}[1]{\todo[color=yellow]{#1}}
\newcommand{\tododv}[1]{\todo[color=yellow]{#1}}
\newcommand{\tododvab}[1]{\todo[color=green]{#1}}
\newcommand{\todoSig}[1]{\todo[inline, color=red]{#1}}
\newcommand{\newContent}[1]{\todo[color=yellow]{New: #1}}
\newcommand{\figscale}[0]{.4}


\newcommand{\ite}[3]{#1\mathrel{\textsf{?}}#2 \mathrel{\textsf{\textbf{:}}} #3}

\newcommand{\fctName}[1]{\textbf{#1}}
\newcommand{\trueText}{\textit{true}}

\newcommand{\checkStyle}[1]{\textbf{\texttt{#1}}}
\newcommand{\emf}[1]{\textit{#1}}

\newcommand{\objVal}[1]{\textsf{#1}}
\newcommand{\objSMT}[1]{\textsc{#1}}

\newcommand{\fofaf}{FAF}
\newcommand{\objectPred}[1]{object(\termConstant{#1})}
\newcommand{\object}{object}

\newcommand{\tptpStyle}[1]{\texttt{#1}}
\newcommand{\solver}[1]{{\small\textsf{#1}}}

\newcommand{\phaseName}{\underline{\textsc{Phase}}}
\newcommand{\phase}[1]{\underline{\textsc{Phase {#1}}}}

\newcommand{\allModels}[1]{\mathcal{M}_\mathit{#1}}

\newcommand{\fwd}[0]{\mathit{ref}}
\newcommand{\refinement}[2]{#1\sqsubseteq#2}

\newcommand{\dslInterpretationDefinition}[1]{\mathcal{I}_{#1}}
\newcommand{\dslInstanceOfI}[3]{\interpretation{\dslInstanceOf{#1}{#2}}{#3}}
\newcommand{\dslInReferenceI}[4]{\interpretation{\dslStyle{#1}(#2,#3)}{#4}}

\newcommand{\nonTerminal}[1]{\langle {#1} \rangle}







\newcommand{\yes}{\newmoon}
\newcommand{\maybe}{\LEFTcircle}
\newcommand{\no}{\fullmoon}
\newcommand{\ok}{$\checkmark$\xspace}


\newcommand{\McGillCrest}[2]{\resizebox{#1}{#2}{\includegraphics{McGillcrest}}}

\newcommand{\ie}{{\em i.e.,\ }}
\newcommand{\eg}{{\em e.g.,\ }}
\newcommand{\etc}{{\em ,\ etc.\ }}
\newcommand{\etal}{et al.\xspace}
\newcommand{\wrt}{w.r.t\xspace}
\newcommand*{\phd}{Ph.D.\xspace}

\newcommand{\infigure}[1]{{\textsf{#1}}}

\newcounter{questioncounter}

\newcounter{conCounter}
\newcounter{resultcounter2}
\newcounter{resultcounter3}

\newcommand{\dslDomainSet}[0]{\mathcal{D}}   
\newcommand{\dslAttributeSet}[0]{\mathcal{V}}

\newcommand{\dslDomainSetI}[1]{\dslDomainSet_{#1}}
\newcommand{\dslAttributeSetI}[1]{\dslAttributeSet_{#1}}
\newcommand{\rigorous}{strict}

\newcommand{\vg}{\textbf{VG}}

\newcommand{\exampleautorefname}{Example}
\renewcommand{\algorithmautorefname}{Algorithm}
\newcommand{\definitionautorefname}{Definition}
\newcommand{\lemmaautorefname}{Lemma}
\newcommand{\apdthmautorefname}{Theorem}
\renewcommand{\subsubsectionautorefname}{Section}

\def\sectionautorefname{Section}
\def\subsectionautorefname{Section}









 

\newcommand{\chunk}[2]{%
	\fcolorbox{black}{yellow}{\bfseries\sffamily\scriptsize#1}%
   {$\blacktriangleright$#2$\blacktriangleleft$}%
}

\newcommand{\percy}[1]{\chunk{Percy}{{\textcolor{purple}{\textsl{#1}}}}}

\newcommand{\gunter}[1]{\chunk{Gunter}{{\textcolor{purple}{\textsl{#1}}}}}

\newcommand{\daniel}[1]{\chunk{Daniel}{{\textcolor{purple}{\textsl{#1}}}}}

\title[Certifying Robustness of Graph Convolutional Networks for Node Perturbation]{Certifying Robustness of Graph Convolutional Networks for Node Perturbation with \\Polyhedra Abstract Interpretation}


\author*[1]{\fnm{Boqi} \sur{Chen}}\email{boqi.chen@mail.mcgill.ca}

\author[2]{\fnm{Kristóf} \sur{Marussy}}\email{marussy@mit.bme.hu}

\author[2]{\fnm{Oszkár} \sur{Semeráth}}\email{semerath@mit.bme.hu}

\author[1]{\fnm{Gunter} \sur{Mussbacher}}\email{gunter.mussbacher@mcgill.ca}

\author[1,2,3]{\fnm{Dániel} \sur{Varró}}\email{daniel.varro@liu.se}

\affil[1]{\orgname{McGill University}, \orgaddress{\city{Montreal}, \country{Canada}}}

\affil[2]{\orgname{Budapest University of Technology and Economics}, \orgaddress{\city{Budapest}, \country{Hundary}}}

\affil[3]{\orgname{Linköping Universit}, \orgaddress{\city{Linköping}, \country{Sweden}}}




\abstract{Graph convolutional neural networks (GCNs) are powerful tools for learning graph-based knowledge representations from training data. 
However, they are vulnerable to small perturbations in the input graph, which makes them susceptible to input faults or adversarial attacks. 
This poses a significant problem for GCNs intended to be used in critical applications, which need to provide certifiably robust services even in the presence of adversarial perturbations. 
We propose an improved GCN robustness certification technique for node classification in the presence of node feature perturbations.
We introduce a novel polyhedra-based abstract interpretation approach to tackle specific challenges of graph data and provide tight upper and lower bounds for the robustness of the GCN.
Experiments show that our approach simultaneously improves the tightness of robustness bounds as well as the runtime performance of certification. Moreover, 
our method can be used during training to further improve the robustness of GCNs.}

\keywords{graph neural networks, robustness certification, graph representation learning, node classification}



\maketitle

\section{Introduction}

\noindent\textbf{Context.} Recent advances in graph neural networks, particularly graph convolutional neural networks (GCNs) have significantly expanded their deployment in knowledge-intensive applications \cite{Yu-KBS2020,sun-icus2020,dettmers2018convolutional} where structured semantic data is represented as knowledge graphs. GCNs leverage existing data (e.g., samples of nodes and edges) to infer further insights from the underlying knowledge base.  


A popular graph inference task is \emph{node classification} which aims to classify nodes within graph-based representations. For example, a GCN can be trained to identify predatory conferences within a publication database based on the attributes (e.g., dates and page length) and the connections of nodes (e.g., authorship, papers, and citations).

However, GCNs are susceptible to \textit{adversarial examples} \cite{sun2022adversarial,GNNBook-ch8-gunnemann}. These are carefully selected small perturbations that can change a GCN's output for a specific node. In the context of publication graph analysis, an author could manipulate the classification of their own paper by carefully choosing a publication date, leading the underlying GCN model to fail.

Consequently, ensuring \emph{the robustness of a GCN} is important \cite{rgcn-intro2021}, especially in critical applications \cite{sun-icus2020}. For that purpose, a GCN needs to be ensured to tolerate some (intentionally or accidentally) mislabelled attributes both in training sets or runtime inputs. For example, a GCN for detecting predatory conferences should be robust, so a few changed publication dates should not change its output.

\noindent\textbf{Problem statement.}
\emph{Certifying the robustness of a GCN} \cite{zugner2019certifiable} aims to ensure that no adversarial examples exist (within a designated input range) for a GCN. Without such guarantees, small perturbations could mislead the GCN. In some cases, certification may even refute the GCN robustness by providing an appropriate counterexample.

Unfortunately, computing the exact robustness of even the simplest neural network poses a significant challenge \cite{salzer2021reachability}. As such, existing work largely focuses on estimating the lower or upper limits of GCN robustness \cite{liu2020abstract,zugner2019certifiable}. Yet, these methods can be imprecise or have poor runtime performance, limiting their practical use. 

Abstract interpretation \cite{cousot1977abstract} is a mathematical framework originally developed for program analysis. It has recently been adapted to the certification of regular neural networks with non-relational continuous inputs \cite{singh2019abstract,urban2020perfectly}, leading to enhanced precision and runtime performance in the certification tasks. This technique has also been applied to certify the robustness of GCNs \cite{liu2020abstract}. However, their certification approach solely relies only on the interval abstract domain, which has been shown to be less precise compared to other, more expressive abstract domains \cite{singh2018fast}. Moreover, it fails to capture the relations between node features across different nodes, a critical aspect of the message-passing mechanism in GCNs.

\noindent\textbf{Objectives.}
The main objective of this paper is to improve robustness certification for graph convolutional neural networks. 
For this purpose, we either provide a certificate that formally guarantees the robustness of the GCN or derive a counterexample that demonstrates the GCN's vulnerability within the desired perturbation limit.

\noindent\textbf{Contributions.}
We present a novel framework based on polyhedra abstract interpretation, a more expressive abstract domain, and tailored to certify the robustness of GCN classifiers specified for node classification tasks on a graph. The framework can leverage GPU accelerations to efficiently provide a precise robustness certificate or a potential adversarial example. Moreover, the certification operations are reversible and differentiable, facilitating robust training to improve the robustness of the GCNs. In particular:

\begin{itemize}[noitemsep,topsep=0pt]
    \item We propose a novel certification approach by defining a polyhedra abstract interpretation for graphs to provide an \textit{over-approximation} for the output of a GCN within a given feature perturbation limit (space). 
    \item With the certification approach, we propose a \textit{pair of} novel certifiers that can provide \textit{both lower and upper bound} for GCN robustness by estimating the impact of a node feature on the output. Besides node-level robustness, we also demonstrate how our approach can derive the \textit{collective robustness} of the entire graph.
    \item We propose the \emph{uncertainty region as a new metric} to evaluate the tightness of a pair of GCN certifiers in a \textit{holistic} way using a perturbation space range.
    \item We carry out \emph{experimental evaluation} on three popular node classification datasets to compare our approach with existing certification approaches on certification tightness, runtime performance, and robust training. 
\end{itemize}


\noindent\textbf{Added value.}
Though abstract interpretation is actively used for certifying other neural networks, the unique attributes of graph data in GCNs pose distinct challenges. These challenges arise from the discrete input space and iterative aggregation of data across neighboring nodes, preventing the direct adaptation of existing techniques to GCNs.

Compared to current GCN certification approaches, our approach distinctly improves both the tightness and runtime efficiency of the certification within a given node feature perturbation limit. 
Our polyhedra abstract interpretation approach determines both upper and lower robustness bounds for GCN and uses reversible and differentiable matrix operations to improve scalability. This leads to faster certification for given GCN classifiers and enables GCN certification for larger graphs.  Furthermore, it can also be used for robust training to improve the robustness of GCNs.


\section{Background}
\label{sec:background}

\subsection{Graph neural networks}
\textbf{Node classification} assigns a label (from a set $\labelSet{}$ of predefined labels) to each node in a graph.
Let $\graph{} = (A, X)$  be an attributed graph with $\numNodes{}$ nodes defined by an adjacency matrix $A \in \mathbb{\binaryDomain{}}^{(\numNodes{}, \numNodes){}}$ ($\mathbb{\binaryDomain{}}=\{0, 1\}$) and a feature matrix with $\numFeatures$ features (attributes) for each node $X \in \mathbb{\realDomain{}}^{(\numNodes, \numFeatures)}$. Without loss of generality, we assume that nodes are ordered from 1 to n and use the index to represent a node. Given the correct labels for the first $k$ nodes ($\lab_{i} \in \labelSet{}, i \le k$) as ground truth, 
\emph{node classification} aims to assign appropriate labels  to the remaining nodes $k+1, \dots, \numNodes$ in the graph
 ($\lab_{j} \in \labelSet{}, k < j \le \numNodes$).
 

\noindent
\textbf{Graph neural networks (GNN)}
support learning on graphs, which have strong input data dependency caused by edges. This paper focuses on the semi-supervised node classification problem using graph convolutional neural networks (GCNs) \cite{kipf2016semi}. 
\chadded{Following prior work \cite{zugner2019certifiable,liu2020abstract,schuchardt2023collective}, this paper focuses on homogeneous graphs, where there is only a single type of nodes and a single type of edges.}

Given a graph $\graph=(\adj, X)$,
GCNs operate on a \textit{normalized adjacency matrix} \(\Tilde{\adj}=D^{-\frac{1}{2}}(\adj + \identity_\numNodes)D^{-\frac{1}{2}}\) where $\identity_\numNodes$ is an identity matrix of size $\numNodes$ and \(D=\operatorname{diag}\bigl(\sum_j{(\adj + \identity_\numNodes)_{:,j}}\bigr)\) is the diagonal degree matrix for each node on the graph. A \textit{GCN layer} consists of three phases. (1) The graph convolution \(\gc(H) := \Tilde{\adj} H\) to aggregate features of neighboring nodes. (2) The linear layer \(\linear{\weight_l, b_l}(H):= H \weight_l + b_l\) where $\weight_l$ and $b_l$ are the layer-specific trainable weight matrix and bias vector, respectively. (3) Finally $\relu(H):=max(0, H)$ is the \emph{non-linear activation function}. Given $\Tilde{A}$, a set of labels $\labelSet$, and the input feature matrix $H_l$, a GCN with $z$ layers can be defined recursively as  $\latentFeatures_{l+1} = \relu(\linear{\weight_l, b_l}(\gc(H_l))), 0 \le l < z$ such that $\latentFeatures_0=\nodeF$ is the input matrix and $\latentFeatures_\numLayers\in \realDomain^{(\numNodes, |\labelSet|)}$ is the output score matrix. We refer to the output of a GCN classifier \(\classifier\) as 
$H_\numLayers = \classifier(A, X)$. The final label for each node is obtained by selecting the label with the highest score per node: $argmax(\latentFeatures_z)$.

While our experiments focus on binary node features $(\nodeF{} \in \binaryDomain{}^{(\numNodes{}, \numFeatures{})})$ with ReLU activations, the approach naturally extends to continuous node features and other non-linear activations through linear relaxation \cite{singh2019abstract}.

\subsection{Robustness of GNNs}
\noindent
\textbf{Perturbation model:}
We investigate the robustness of a classifier by making assumptions about potential perturbations in the graph. Such perturbations may be caused by either uncertain measurements or adversarial changes of an attacker. 
Following previous work \cite{zugner2019certifiable,liu2020abstract}, we define node feature perturbations as flips of some (binary) node features limited by two types of change constraints. 
\emph{Local limit} $\localp$ constrains the total number of feature changes for a node $i$, and \emph{global limit} $\globalp$ states the maximum number of feature changes for all nodes. Then, the \textit{perturbation space} can be defined as all feature matrices $\nodeF{}'$ within the two limits from the original feature matrix $\nodeF$: $\perturb{\localp}{\globalp}{\nodeF{}}=\{\nodeF{}' \mid (\forall_i \sum_j(|\nodeF{}'-\nodeF{}|)_{i,j} \le \localp) \land \sum_{i,j}(|\nodeF{}'-\nodeF{}|)_{i,j}\le\globalp\}$. When the context is clear, we omit $\localp, \globalp$ in the notation for simplicity and use $\perturba{\nodeF{}}$.



\noindent
\textbf{Attack model:}
Given a perturbation budget, this paper is centered around an attack model where the main objective of the perturbation is to modify the classification outcomes of selected target nodes. Let $R_{\classifier}(\graph, \nodeF{}', i)$ be an indicator function on whether the prediction of a node is changed for some feature matrix $\nodeF{}'$ compared to the original feature matrix $\nodeF{}$. It gives $1$ if the prediction remains consistent for the perturbed feature matrix $\nodeF{}'$ and $0$ otherwise. \chadded{Formally, let $c$ be the original label prediction for the feature matrix $X$}, $R_{\classifier}(\graph, \nodeF{}', i)$ is defined as:
\begin{equation}
\text{\(\begin{aligned}
   R_{\classifier}(\graph, \nodeF{}', i) :=&\; \argmax{\classifier(\adj{}, \nodeF{})_\node{}} = \argmax{\classifier(\adj{}, \nodeF{}')_\node{}} \\
       \; \equiv & \forall c' \ne c, \bigl[ \delta_{i, c, c'} \bigr] > 0\text{,}\\
\end{aligned}\)}\label{equ:node_robustness_min}
\end{equation}
where $\delta_{i, c, c'} := \classifier{}(\adj{}, \nodeF{}')_{\node{},\lab{}} - \classifier{}(\adj{}, \nodeF{}')_{\node{}, {\lab{}'}}$ and $\classifier{}(\cdot, \cdot)_{\node{}, \lab{}}$ denotes the output score of label $\lab{}$ at node $\node{}$. The goal of perturbation is to identify an adversarial matrix $\nodeF{}' \in \perturba{\nodeF{}}$ where the prediction of the node is changed: $R_{\classifier}(\graph, \nodeF{}', i) = 0$. 

This attack model can be further categorized into two distinct attack types. The \textbf{single-node} attack model assumes the attacker can individually perturb the graph for each target node, i.e., $\nodeF{}'$ may vary for each node. In this scenario, the perturbation space is uniquely defined for each target node, ensuring that perturbations to one target node do not affect another target node. However, in many practical applications, perturbations for individual nodes may not be independent, given that all nodes are frequently considered collectively as a graph \cite{schuchardt2023collective}. The \textbf{collective} attack model revises this perspective by dropping the assumption on perturbation independence. It operates under the assumption that the graph can only be perturbed \textit{once} for all target nodes, thus the result $\nodeF{}'$ stays the same for all nodes. 

\noindent
\textbf{Single-node robustness:}
In single-node robustness, one assumes each node has a perturbation space independent of other nodes. The output of a classifier $\classifier$ for node $i$ is \textit{robust} against a perturbation space $\perturba{\nodeF{}}$ if no perturbed feature matrix $\nodeF{}'$ in $\perturba{\nodeF{}}$ can change  classification of $\node{}$ by $\classifier{}$. We call such node a \textit{robust node} for $\classifier$. 
\emph{Node-level robustness} $\nodeR{\node{}} \in \binaryDomain{}$ can be defined as:
\begin{equation}
\text{\(\begin{aligned}
    \nodeR{\node{}} :=&\; \forall \nodeF{}' \in \perturba{\nodeF{}}\colon  R_{\classifier}(\graph, \nodeF{}', i) > 0 \;
    \equiv \bigl[ min_{\nodeF{}' \in \perturba{\nodeF{}}, \lab{}' \ne \lab{}} \; \delta_{i, c, c'} \bigr] > 0\text{,}\\
\end{aligned}\)}
\end{equation}
 
For a robust node, the output scores for all other labels are always lower than the score of the original GCN label for any feature changes inside the perturbation space.
The (graph-level) \textit{robustness} $\graphR{} \in [0, 1]$ of a classifier $\classifier$ on a graph $\graph{}$ with $n$ nodes against a perturbation space $\perturba{\nodeF{}}$ is defined as the proportion of robust nodes: 

\begin{equation}
\graphR{}=\frac{\sum_{\node{}=1}^{\numNodes}{\nodeR{i}}}{\numNodes{}}
\end{equation}



A perturbed matrix $\nodeF{}' \in \perturba{\nodeF{}}$ is a \textit{counterexample} (or adversarial example) for the robustness of a GCN classifier $\classifier{}$ on node $\node$ if the classification of node $\node{}$ in $\graph' = (\adj, \nodeF{}')$ differs from the classification of the same node in $\graph = (\adj, \nodeF{})$: $\argmax{\classifier(\adj{}, \nodeF{})_\node{}} \ne \argmax{\classifier{}(\adj{}, \nodeF{}')_\node{}}$.

\noindent
\textbf{Collective robustness:}
In collective robustness, a single shared perturbation space is used for all nodes on the graph. Consequently, only \textit{graph-level} robustness is defined for this attack model. This graph-level robustness is categorized by the minimum ratio of nodes whose label stays unchanged by any feature matrix within the perturbation space:
\begin{equation}
\text{\(\begin{aligned}
    \graphR{}_{\text{collect}} := min_{\nodeF{}' \in \perturba{\nodeF{}'}} \frac{\sum_{i=1}^{n}{R_{\classifier}(\graph, \nodeF{}', i)}}{n}\\
\end{aligned}\)}
\end{equation}

Throughout this paper, we dominantly focus on the \textit{single-node} robustness certification, unless otherwise specified. Nevertheless, our method on single-node robustness certification can be used to determine the \textit{collective robustness} through a subsequent post-processing approach, as shown in \autoref{sec:collective-certification}.

\subsection{Robustness certification}
Robustness certification aims to provide formal proof that a GCN is not sensitive to small perturbations of the input graph. 
Given a GCN classifier $\classifier$ with label set $\labelSet{}$ for graph $\graph{}=(\adj{}, \nodeF{})$ with $\numNodes{}$ nodes, and a perturbation space $\perturba{\nodeF{}}$,  \textit{GCN robustness certification} provides a \textit{certifier} $\certifier{}: (\adj{}, \perturba{\nodeF{}}, \classifier{}) \to \realDomain{}^\numNodes{}$. The output of the certifier is a \emph{judgment} for each node that estimates the minimum difference (in \autoref{equ:node_robustness_min}) between the original label and other labels for all inputs within the perturbation space. 
We define the \textit{node-level certified 
 robustness as} $\nodeRC{\node{}}=\certifier(\adj{}, \perturba{\nodeF{}}, \classifier{})_\node{} > 0$. Similarly, the \emph{graph-level certified robustness} of a $\classifier$ for $\graph{}$ using $\certifier{}$ is the ratio of certified nodes  defined as:
 
\begin{equation}
\graphRC{}=\frac{\sum_{\node{}=1}^{\numNodes}{\nodeRC{i}}}{\numNodes{}}
\end{equation}

However, it is intractable to enumerate all possible perturbations $\nodeF{}'$ in $\perturba{\nodeF{}}$ to analytically calculate \emph{exact} robustness $\nodeR{\node{}}$  for each node $i$ \cite{zugner2019certifiable}. 
Hence, the adoption of some approximation methods is necessary. Techniques such as randomized smoothing are designed to estimate this value, providing a probabilistic approximation on a \textit{smoothed} version of the GCN \cite{wang2021certified}. Nevertheless, in practice, it is usually desirable to estimate the robustness of the \textit{original} GCN with some \textit{absolute} attributes.
Consequently, this paper aims to create a certifier $\certifier{}$ that provides an exact (tight) lower and upper bound of single-node robustness. Since collective certification is defined similarly for collective robustness \cite{schuchardt2023collective}, we omit the details for conciseness. We define the concept of sound and complete certifiers similarly to others \cite{liu2021algorithms}. 

\begin{definition}[Sound and complete certifiers]
\label{def:soundness}

A certifier $\certifier$ is \emph{sound} if whenever it flags a node $i$ as a robust node, the node will indeed 
 be robust against the perturbation space $\perturba{\nodeF{}}$: $\nodeRC{\node{}} \implies \nodeR{\node{}}$. A certifier $\certifier$ is \emph{complete} if all robust nodes are flagged for a given perturbation space $\perturba{\nodeF{}}$: $\nodeR{\node{}} \implies \nodeRC{\node{}}$.
\end{definition}

 
A sound certifier provides a lower bound for graph-level robustness, i.e., $\graphRC{} \le \graphR{}$.
A complete certifier  gives an upper bound for robustness, i.e., $\graphRC{} \ge \graphR{}$.
A \textit{sound and complete certifier} judges each node correctly, i.e., the lower and upper bounds are the same. A certification approach may provide a pair of certifiers: a sound certifier for a lower bound $\graphRLower{}$ and a complete certifier for an upper bound $\graphRUp{}$.



\section{Approach}
\label{sec:approach}

\paragraph{Problem description}
Given a GCN node classifier $\classifier$ that assigns a label to a graph node, an input graph $\graph{}$ and perturbation limit $(\localp, \globalp)$,  \textit{GCN robustness certification} provides \textit{a pair of certifiers} that estimate the lower bound robustness $\graphRLower{}$ and upper bound robustness $\graphRUp{}$ of $\classifier$, respectively. \chadded{The objective of robust certification is to tighten the bounds by minimizing the gap between the upper and lower robustness estimates, $\graphRUp{} - \graphRLower{}$.}

\begin{figure}
    \centering
    \includegraphics[width=\linewidth]{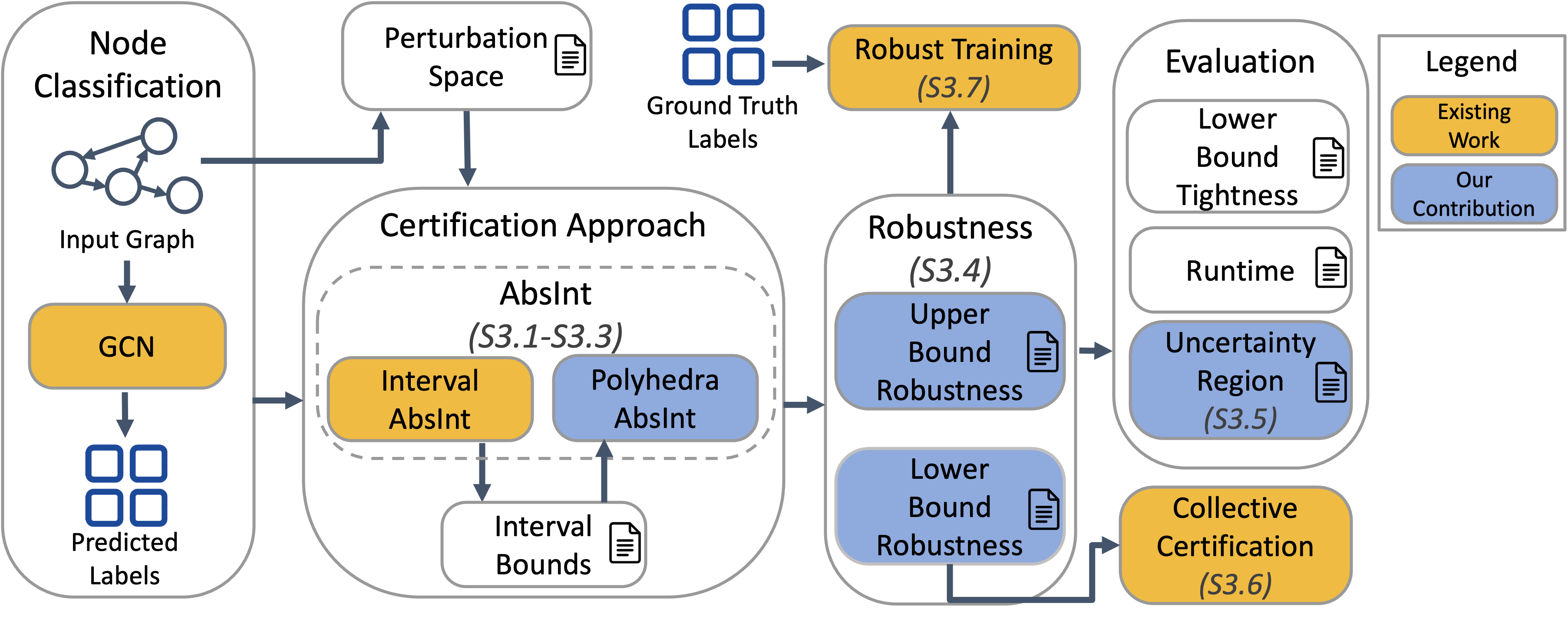}
    \caption{\chadded{Overview of the approach components and the corresponding subsections}}
    \label{fig:overview}
\end{figure}

\paragraph{Approach overview}
\autoref{fig:overview} shows an overview of our approach, which combines two abstract interpretation (\absint{}) techniques into a  unified framework (\autoref{sec:abs-int-gcn}). In the figure, each rounded rectangle symbolizes a process in the framework, and those combined with a document icon signify the input and/or output data associated with a process. Input graphs are characterized by circles and arrows, and labels are denoted with square icons. Work from previous studies are distinguished by an orange label, while our contributions are marked in blue.
\textit{Interval \absint{}} creates \textit{interval bounds} for each latent feature value of the GCN on the graph. These bounds are used in \textit{polyhedra \absint{}} (Sections~\ref{sec:polyhedra-abs-int}-\ref{sec:abst-oper}) to derive tighter \emph{symbolic bounds} of the values. The resulting bounds can be used to calculate the \textit{lower} (sound) and \textit{upper} (complete) \textit{bounds} of the node and graph-level GCN \textit{robustness} (\autoref{sec:rbst-cert}). We also demonstrate using our approach how to derive \textit{collective certification} (\autoref{sec:collective-certification}).  Moreover, the certification process can be used to improve the robustness of the GCN by \textit{robust training} (\autoref{sec:robust-training}) by providing the \textit{ground truth labels}. 

In \autoref{sec:evaluation}, we evaluate the approach in terms of \textit{runtime} performance and \textit{tightness} of certification with the help of a new metric \textit{uncertainty region} (defined in \autoref{sec:uncertainty_region}). We include all proofs of theorems in the appendix.  \autoref{tab:notions} summarizes the most important notations used in the following sections.

\begin{table}[]
    \centering
    \makebox[\linewidth][c]{
    \begin{tabular}{|c|c|}
        \hline
        $\wp(S)$ & The power set of a set $S$ \\
        \hline
        $\concreteDomain$ & Concrete domain \\
         \hline 
        $\abstractDomain$ & Abstract domain  \\
         \hline
        $\overline{\operation}\colon \concreteDomain \to \concreteDomain$ & The set-valued operation of a function $T$\\
        \hline
        $\operation^\#: \abstractDomain \to \abstractDomain$ & Abstract operation of a function $\operation$ operates on the abstract domain \\
        \hline
        $\texttt{x}_{\node{}, \feature{}}$& Symbolic variable in the polyhedra domain for the input feature $X_{\node,\feature}$ \\
        \hline
        $\mathcal{X}, \map_i$ & The set of all symbolic variables and the set of symbolic variables for node $i$ \\
        \hline
        \scalebox{0.85}{$\polyLowExp{\node{}}, \polyLowConst{\node{}}$}& The coefficient matrix and constant vector for the lower bound of symbolic variables of node $i$ \\ 
        \hline
        \scalebox{0.85}{$\polyUpExp{\node{}}, \polyUpConst{\node{}}$}& The coefficient matrix and constant vector for the upper bound of symbolic variables of node $i$ \\
        \hline
        $\coef{}_\node{}$ & Set of inequalities bonding variables for node $i$ \\
        \hline
        $\absele$ & An abstract element\\
        \hline
    \end{tabular}
    }
    \caption{Summary of notions in the approach}
    \label{tab:notions}
\end{table}

\subsection{\absint{} for GCN certification}
\label{sec:abs-int-gcn}
The core technique behind our approach is abstract interpretation (\absint{})  \cite{cousot1977abstract}, which is a mathematical framework initially designed to reason about the behavior of computer programs. Recently, it has also been adapted to neural network certification of non-relational continuous inputs \cite{urban2020perfectly,singh2019abstract}. Here, we first define a \emph{general abstract interpretation framework for certifying GCN node classifiers}.

Abstract interpretation relies on an \emph{abstract domain} $\abstractDomain$~\cite{cousot1977abstract}, which can be any numerical or relational domain that provides a sound abstraction of elements in the \emph{concrete domain} $\concreteDomain$. Moreover, \absint{} defines sound abstractions of operations in the concrete domain.

For \emph{\absint{} of GCN certification}, the concrete domain captures the perturbations space $\perturba{\nodeF{}}$ and the set of possible latent features $\latentFeatures_l$ produced by the layers in the GCN over the perturbations space. Then \absint{} aims to produce an approximation of the concrete domain (a set of feature matrices) with an abstract domain by sound abstraction of GCN operations (e.g., graph convolution, ReLU). 

\subsubsection{Concrete domain and operations}
Formally, the concrete domain of a GCN on a graph with $\numNodes$ nodes can be defined as sets of binary and real-valued matrices  
\(\concreteDomain{} = \{\wp(\mathbb{B}^{(n, m_0)})\}\cup \{ \wp(\mathbb{R}^{(n, m_l)}) \mid 0<l\le \numLayers\}\) where 
$\numLayers$ is the number of layers in the GCN, $m_l$ is the number of node features at layer $l$ and
$\wp(\set{S})$ represents the power set of set $\set{S}$. 
For any operation \(\operation\) defined over such matrices, the corresponding \emph{set valued} operation \(\overline{\operation}\colon \concreteDomain{} \to \concreteDomain{}\) computes all possible resulting matrices from the input set $\set{S}$: \(\overline{\operation}(\concreteElement) = \{ \operation(H) \mid H \in \concreteElement \}\).
The set of concrete operations~\(\operationSet{}\) for a GCN includes \(\overline{\gc}\) for \emph{graph convolution} \(\gc\), \(\overline{\relu}\) for the \emph{non-linear} \(\relu\), and \(\overline{\linear{\weight_l, b_l}}\) for the \emph{fully connected} operation \(\linear{\weight_l, b_l}\).

\subsubsection{Abstract interpretation (domain and operations)} 
The \emph{abstract interpretation} over a GCN is a tuple $(\abstractDomain{}, \abstractFunc{}, \concreteFunc{}, \operationSet{}^\#)$ with \emph{abstraction function}~\(\abstractFunc\colon \concreteDomain \to \abstractDomain\), \emph{concretization function} \(\concreteFunc\colon \abstractDomain \to \concreteDomain\),
and a set of \emph{abstract operations} \(\operationSet{}^\#\).
Functions \(\abstractFunc\) and \(\concreteFunc\) transform elements between the concrete domain and the abstract domain. Since \(\concreteFunc\) and  
\(\abstractFunc\) mutually determine each other \cite{cousot1977abstract}, we define \emph{input abstraction} and then use \(\concreteFunc\) in the rest of the paper.
For each concrete operation \(\overline{\operation} \in \operationSet\), we use $\operation^\#$ to denote the corresponding abstract operation that \textit{over-approximates} \(\overline{\operation}\)~(i.e., \(\overline{\operation}(\concreteFunc(\absele)) \subseteq \concreteFunc(\operation^\# (\absele))\) for all \(\absele \in \abstractDomain\)).
For any GCN $\classifier{}$, we use $\overline{\classifier{}}$ to represent the GCN with concrete operations from $\operationSet{}$, and we use $\classifier{}^\#$ to represent the abstract GCN with the abstract operations from $\operationSet{}^\#$.

\begin{theorem}[Over-approximation of \absint{}]
\label{thm:ai_soundness}
Any abstract interpretation $(\abstractDomain{}, \concreteFunc{}, \operationSet{}^\#)$ defined for GCN as above over-approximates the behavior of the GCN classifier, i.e., \(\overline{\classifier}(A, \perturba{X}) = \{ \classifier(A, X') \mid X' \in \perturba{X} \} \subseteq \concreteFunc(\classifier^\# (\absele))\) for all abstract elements \(\absele \in \abstractDomain\) with \(\perturba{X} \subseteq \concreteFunc(\absele)\).
\end{theorem}
\chadded{The proof for \autoref{thm:ai_soundness} is presented in Appendix \ref{append:proof_ai_soundness}}


One instantiation of the general \absint{} framework for GCN certification is interval \absint{} \cite{liu2020abstract}, which 
(1) bounds each latent feature with an interval, and (2) uses interval arithmetic as abstract operations. 
However, interval \absint{} can be imprecise as it does not capture the relations between node features. 
A novel contribution of our paper is to define polyhedra-based \absint{} for GCN certification (in \autoref{sec:polyhedra-abs-int}-\ref{sec:abst-oper}), which incorporates such relations.

\subsection{Efficient polyhedra \absint{}}
\label{sec:polyhedra-abs-int}
\subsubsection{Challenges}
The \textit{polyhedra abstract domain} \cite{cousot1978automatic} bounds variables by linear inequalities to offer a precise domain to capture relations between node features. Specifically, we tackle the following challenges when designing the approach:

\textbf{Scalability:} The \emph{number of inequalities may grow exponentially with each layer of the GCN}. We mitigate this issue similarly to other neural networks \cite{singh2019abstract} by keeping only one upper-bound inequality and one lower-bound inequality to improve certification efficiency. 

\textbf{Tightness:}
While existing work uses linear bounds to approximate non-linear layers (like $\relu$) without increasing the number of variables, we \emph{provide an analytical solution for the tightest approximation}. 

\textbf{Discrete data:} While our approach is inspired by abstract interpretation-based certification methods for neural networks on non-relational and continuous data, \emph{a challenging and novel aspect of our work is to handle an interdependent and discrete input space}. Due to the relational nature of graphs, the GCN's output of one node depends on other nodes. However, naively incorporating all neighboring nodes in the certification will significantly increase the runtime due to unnecessary computations. 

\textbf{Efficiency:} Most abstract interpretation methods are performed on CPUs, which makes them hard to scale up for GCNs of large real-world graphs. To improve the scalability, we combine the polyhedra and interval \absint{} approach to balance speed and precision. At the same time, \emph{our abstract GCN operators in the polyhedra domain are defined with reversible and differentiable matrix operations}, enabling efficient GPU-accelerated implementation when used on large-scale graphs. Since our abstract operators are reversible, we design a backward process such that a neighbor's feature is only added to the computation when needed. Moreover, our abstract operators are defined for each node independently while taking the neighbors into consideration. This enables efficient subgraph batching based on neighbors during the certification process, which further improves the certification speed. Finally, the same abstraction can also be used to provide counterexample candidates.

\subsubsection{Overview}
 Building on the elements outlined for an abstract interpretation approach (see \autoref{sec:abs-int-gcn}), our method begins with an \textit{abstract domain}, which involves explaining how polyhedra can be used to overapproximate a set of features. Subsequently, we describe the \textit{input abstraction} process, which transforms a set of input features into their overapproximation within the abstract domain. Similarly, the \textit{concretization} process describes how the abstract element can be mapped into a set of elements in the concrete domain. The next stage includes a series of \textit{abstract operations} designed to overapproximate the neural network's layers. Specifically, for GCNs, these operations include the fully connected layer ($\linear{}{}$), graph convolution ($\gc{}$), and the ReLU activation function ($ReLU$). 

\subsubsection{Polyhedra \absint{}}
Let us assign a symbolic variable $\texttt{x}_{\node{}, \feature{}}$ to an input feature $X_{\node,\feature}$ of node $\node$. Further, we define $\mathcal{X}$ to be the set of all symbolic variables from the graph and $\map_i$ as the set of variables for node $\node$.
For a concrete domain of feature vectors $\concreteElement{}\in \concreteDomain{}$, the abstract domain for a GCN can be defined as a collection of convex polyhedra, each bounding the feature values of a graph node. We represent a polyhedron by a system of inequalities \cite{cousot1978automatic}.
Since the neighbors of a node contribute to a prediction by the $\gc{}$ operator, the variables \(\map_i \subseteq \nodeVar\) used for feature bounds of node $\node{}$ at layer $\layer{}$ can include both features of node $\node{}$ and its $\layer{}$-hop neighbors.
 

\noindent\textbf{Abstract domain:} We define the abstract domain for graph features as a set of tuples \(\absele = (\coef{}_\node{}, \map{}_\node)_{i = 1}^n  \in \abstractDomain\), where $\coef{}_\node{}$ is the set of inequalities used to bound latent variables for node $\node$: 
\( \coef{}_\node{} = (\polyLowExp{\node{}}, \polyLowConst{\node{}},\) \(\polyUpExp{\node{}}, \polyUpConst{\node{}})\) where \(\polyLowExp{\node{}}\)  and \(\polyUpExp{\node{}}\) are matrices of coefficients of variables \(\map_\node\), and \(\polyLowConst{\node{}}\) and \(\polyUpConst{\node{}}\) are vectors of constants in coefficients \(\coef_i\). Let \(\texttt{h}_{i,j}\) be a symbolic variable associated with the latent feature \(H_{i,j}\) of each node \(i\), we retain the two inequalities:

\noindent{\small\setlength{\abovedisplayskip}{0pt}\setlength{\belowdisplayskip}{0pt}\begin{align*}
    \textstyle\sum_{\texttt{x}_k \in \map{}_i}\! \bigl((\polyLowExp{\node{}})_{j,k} \cdot \texttt{x}_k\bigr) + (\polyLowConst{\node{}})_j &\le \texttt{h}_{i,j}\text,\\
    \textstyle\sum_{\texttt{x}_k \in \map{}_i}\! \bigl((\polyUpExp{\node{}})_{j,k} \cdot \texttt{x}_k\bigr) + (\polyUpConst{\node{}})_j &\ge \texttt{h}_{i,j}\text,
\end{align*}}%
An abstract element \(\absele\) uniformly refers to a set of linear inequalities and the tuples of matrices that define them.


\noindent\textbf{Input abstraction:}
In the input layer, each latent feature variable \(\texttt{h}_{i,j}\) corresponds to a node feature variable \(\texttt{x}_{i,j}\).
To create an initial abstraction for node features \(X\),  each feature variable \(\texttt{x}_{i,j}\) is used as its own lower and upper bound, thus the retained inequalities are \(1 \cdot \texttt{x}_{i,j} + 0 \le \texttt{x}_{i,j}\) and \(1 \cdot \texttt{x}_{i,j} + 0 \ge \texttt{x}_{i,j}\).
This corresponds to the abstract element \(\absele_0 = \bigl((I^{\numFeatures{}_0}, \mathbf{0}^{\numFeatures{}_0}, I^{\numFeatures{}_0}, \mathbf{0}^{\numFeatures{}_0}), \{\texttt{x}_{\node{}, \feature{}}|j<m_0\}\bigr)_{i = 1}^n \in \abstractDomain\), where \chadded{$m_0$ is the number of input features,} $I^{\numFeatures{}_0}$ is an \(m_0 \times m_0\) identity matrix, and $\mathbf{0}^{\numFeatures{}_0}$ is a vector of zeros.

\noindent\textbf{Concretization:}
To compute concretization \(\concreteFunc(\absele)\), we consider both the perturbation \(\perturba{\nodeF{}}\) of node features \(\nodeVar\) and the linear inequalities retained in \(\absele\). Formally, we have \(\concreteFunc(\absele) = \{ H \in \mathbb{R}^{(n,m)} \mid H \vDash \absele \cup \perturba{X}^{\localp,\globalp} \}\) to obtain a set of matrices that simultaneously satisfy ($\vDash$) inequalities in \(\absele\) as well as  the local and global  limits \(\localp, \globalp\).

\noindent\textbf{Abstract operations:}
We define three abstract operations to capture the relations between features from layer $\layer{}$ and the input features. $\linear{\weight{}, \bias}^{\#_P}$  and $\gc{}^{\#_P}$ are linear operations, hence such relations are captured exactly using the polyhedra domain. $\relu{}^{\#_P}$ is a non-linear function hence approximations are required. 
For illustration, we separate the abstraction and transformation for the coefficients ($\coef_i$) and set of variables ($\map_\node$) and only the node-level abstract operations are presented in detail in \autoref{sec:abst-oper}. The abstract operations for the whole abstract element are the aggregation of results for all nodes in the graph, i.e., \(\absele' = (\coef'_i, \map'_i)_{i = 1}^n\). Next, we present each abstract operation in detail.


\subsection{Abstract operations}
\label{sec:abst-oper}

\noindent\textbf{Fully connected:}
$\linear{\weight{}, \bias}^{\#_P}$ keeps all linear relations of input node features and output node features. 
Intuitively, the lower bound of its output can be calculated by multiplying the positive weights by the input lower bound and the negative weights by the input upper bound. The upper bound can be calculated similarly.  
For an input element of node $\node{}$ with coefficients $\coef_i = (\polyLowExp{\node{}}, \polyLowConst{\node{}}, \polyUpExp{\node{}}, \polyUpConst{\node{}})$ and variables $\map{}_\node{}$, the $\linear{\weight{}, \bias}^{\#_P}$  operation can be formally defined as:

\noindent{\setlength{\abovedisplayskip}{0pt}\setlength{\belowdisplayskip}{0pt}\allowdisplaybreaks[4]\small
\begin{align*}
&\linear{\weight{}, \bias}^{\#_P}(\polyLowExp{\node{}}, \polyLowConst{\node{}}, \polyUpExp{\node{}}, \polyUpConst{\node{}}) = (\polyLowExp{\prime\node{}}, \polyLowConst{\prime\node{}}, \polyUpExp{\prime\node{}}, \polyUpConst{\prime\node{}})\text{, where}\\
&\polyLowExp{\prime\node{}} = max(\weight{}^T, 0) \cdot\polyLowExp{\node{}} + min(\weight{}^T, 0)\cdot \polyUpExp{\node{}},\\
&\polyUpExp{\prime\node{}} = max(\weight{}^T, 0)\cdot \polyUpExp{\node{}} + min(\weight{}^T, 0)\cdot \polyLowExp{\node{}},\\
&\polyLowConst{\prime\node{}} = max(\weight{}^T, 0) \cdot\polyLowConst{\node{}} + min(\weight{}^T, 0)\cdot \polyUpConst{\node{}} + \bias,\\
&\polyUpConst{\prime\node{}} = max(\weight{}^T, 0)\cdot \polyUpConst{\node{}} + min(\weight{}^T, 0)\cdot \polyLowConst{\node{}} +\bias,
\end{align*}}%
\chadded{$W^T$ represents the transpose of the weight matrix $W$.} $max(\weight{}^T, 0)$ and $min(\weight{}^T, 0)$ keep positive and negative values in matrix $\weight{}^T$, respectively. Notice that this operation does not introduce any new variables. Thus, the set of node feature variables stays unchanged $\map{}'_\node{} = \map{}_\node{}$.

\noindent\textbf{Graph convolution:}
The graph convolution $\gc{}^{\#_P}$ operation collects variables from node $\node{}$ and its neighbors and combines them into a matrix. Since all entries in the normalized adjacency matrix $\Tilde{\adj}$ are non-negative, we can obtain the upper bound and lower bound by calculating a normalized sum using bounds from neighbors. Since the summation step may introduce new variables from neighbors, we can incorporate the new variables using matrix concatenation. Let $\neighbors{\node{}}{1}$ be the immediate neighbors of node $\node{}$ (including itself);
we define $\gc{}^{\#_P}$ for node $\node{}$ as 


\noindent{\setlength{\abovedisplayskip}{0pt}\setlength{\belowdisplayskip}{0pt}\allowdisplaybreaks[4]\small
\begin{align*}
&\gc{}^{\#_P}(\polyLowExp{\node{}}, \polyLowConst{\node{}}, \polyUpExp{\node{}}, \polyUpConst{\node{}}) = (\polyLowExp{\prime\node{}}, \polyLowConst{\prime\node{}}, \polyUpExp{\prime\node{}}, \polyUpConst{\prime\node{}})\text{, where} \\
&\polyLowExp{\prime\node{}} = [\Tilde{\adj}_{\node{},k}\cdot\polyLowExp{k}]_{k\in \neighbors{\node{}}{1}}\quad \polyUpExp{\prime\node{}} = [\Tilde{\adj}_{\node{},k}\cdot\polyUpExp{k}]_{k\in \neighbors{\node{}}{1}}\\
&\polyLowConst{\prime\node{}} = \textstyle\sum_{k \in \neighbors{\node{}}{1}}(\Tilde{\adj}_{\node{},k}\cdot\polyLowConst{k}) \quad \polyUpConst{\prime\node{}} = \textstyle\sum_{k \in \neighbors{\node{}}{1}}{(\Tilde{\adj}_{\node{},k}\cdot\polyUpConst{k})}
\end{align*}}

\noindent
where $[H_s]_{s \in \set{S}}$ concatenates all matrices $H_s$ defined by $s \in \set{S}$ horizontally. 
In this operation, the set of variables for the new element is the union of all variables from the neighbors: $\map{}'_\node{} = \bigcup_{k\in \neighbors{\node{}}{1}}\map{}_k$. Note that there may be overlapping between the set of variables. In that case, one sums the corresponding columns of the coefficient matrix, and the number of new variables will be reduced. We ignore such cases for conciseness. 

\begin{figure}
    \centering
    \includegraphics[width=\linewidth]{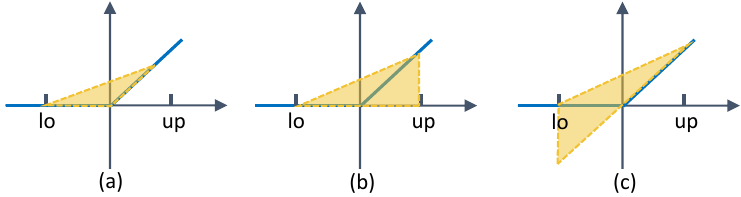}
    \caption{Three possible boundings of the ReLU function: (a) bounding with the minimum area; (b), (c) two possible boundings maintaining one lower bound}
    \label{fig:relu}
\end{figure}

\noindent\textbf{ReLU:}
As the $\relu{}$ function is non-linear, relaxation is needed with  \textit{one} upper bound and lower bound, respectively. Such relaxation of the $\relu{}$ function requires \emph{numerical} upper and lower bounds, but the polyhedra \absint{} only produces \emph{symbolic} upper and lower bounds so far. 
To get the numerical bounds efficiently for each variable, we use the interval bounds \(\low{}\) and \(\high{}\) of the variable calculated by \textit{interval \absint{}} \cite{liu2020abstract}.
The polyhedron bounding of $\texttt{x}' = \relu(\texttt{x})$ function is divided into three cases according to the interval bounds.

(1) If $\low{} \ge 0$, the $\relu{}$ function is equivalent to an identity function; thus we can set $\variable{}'=\variable{}$. (2) If $\high{} < 0$, the $\relu{}$ function sets all negative values to 0, so we have $\variable{}'=0$. (3) If $\low{} < 0 < \high{}$, the $\relu{}$ function is non-linear and we need a convex relaxation. 

\autoref{fig:relu} shows three possible bounds of the $\relu{}$ function (shaded areas represent uncertainties in the bounding). Bounding (a) provides the minimum area using a convex polyhedron, but it violates the requirement of the abstract domain since it has two lower bounds ($\variable{}'=\variable{}$ and $\variable{}'=0$). Boundings (b) and (c) are slightly larger than bounding (a) but only require one lower bound. Hence, we use bounding (b) or (c) for $\relu{}$. In general, the lower bound is $\variable{}'=\lambda \variable{}$ where $\lambda\in[0, 1]$ where $\lambda$ is chosen to  minimize the resulting bounding area. While similar boundings have been presented for certifying other neural networks \cite{singh2019abstract}, our paper also provides an analytical solution for bounds \emph{with the minimum area}.

\begin{theorem}
    \label{thm:relu_bounding}
    Let $\variable{}$ be a variable in interval $[\low{}, \high{}]$ such that $\low{} < 0 < \high{}$. The \textbf{minimum area bounding} of $\relu{}(\variable{})$ with one upper and one lower bound can be defined as:\\
    (1) upper bound: $\variable{}'=\frac{\high{}}{\high{}-\low{}}\variable{} - \frac{\high{} \cdot \low{}}{\high{} - \low{}}$\\
    (2) lower bound: $\variable{}'= \begin{cases}
            \variable{} & \text{if } |\high{}| \ge |\low{}| \\
            0 & \text{otherwise.}
        \end{cases}$
\end{theorem}
\chadded{The proof can be found in Appendix \ref{append:proof_relu_bounding}}. \autoref{thm:relu_bounding} ensures that the area of the bounding is minimized by $\lambda \in \{0, 1\}$ for values $\low{}$ and $\high{}$. 
Let $lo_{\node{}\feature{}}$, $up_{\node{}\feature{}}$ be the minimum and maximum values of  interval Abs\-Int boundings for node $\node{}$ and feature $\feature{}$. 
The $\relu{}^{\#_P}$ operation is defined for each node $\node{}$ and  feature $\feature{}$ in four cases: (i) $lo_{\node{} \feature{}} \ge 0$, (ii) $up_{\node{} \feature{}} < 0$, (iii) $lo_{\node{} \feature{}} < 0 < up_{\node{} \feature{}} \land |up_{\node{} \feature{}}| \ge |lo_{\node{} \feature{}}|$ and (iv) $lo_{\node{} \feature{}} < 0 < up_{\node{} \feature{}} \land |up_{\node{} \feature{}}| < |lo_{\node{} \feature{}}|$:

\begin{align*}
&\relu{}^{\#_P}(\polyLowExp{\node{}}, \polyLowConst{\node{}}, \polyUpExp{\node{}}, \polyUpConst{\node{}}) = (\polyLowExp{\prime\node{}}, \polyLowConst{\prime\node{}}, \polyUpExp{\prime\node{}}, \polyUpConst{\prime\node{}}) \text{, where} \\
&(\polyLowExp{\prime\node{}})_{j,k}, (\polyLowConst{\prime\node{}})_{j} = \begin{cases}
    (\polyLowExp{\node{}})_{j,k}, (\polyLowConst{\node{}})_j & \text{in case (i), (iii)} \\
    0, 0 & \text{in case (ii), (iv)}\\
\end{cases} \\
&(\polyUpExp{\prime\node{}})_{j, k}, (\polyUpConst{\prime\node{}})_{j}=\begin{cases}
    (\polyUpExp{\node{}})_{j,k}, (\polyUpConst{\node{}})_j & \text{in case (i)} \\
    0, 0 & \text{in case (ii)}\\
    s\cdot (\polyUpExp{\node{}})_{j,k}, s \cdot (\polyUpConst{\node{}})_j + t & \text{in case (iii), (iv)}
\end{cases}\\
&\text{where } s = \frac{up_{\node{}\feature{}}}{up_{\node{}\feature{}}-lo_{\node{}\feature{}}} \text{ and } t=-\frac{up_{\node{}\feature{}} \cdot lo_{\node{}\feature{}}}{up_{\node{}\feature{}} - lo_{\node{}\feature{}}} 
\end{align*}

\noindent The $\relu{}^{\#_P}$ operation introduces no new variables; thus $\map{}'_\node{} = \map{}_\node{}$.


\subsection{Robustness certification for GCNs}
\label{sec:rbst-cert}

Certifying node robustness  requires to compute the difference \(\delta_{i, c, c'}\) between the scores \(\classifier(A, X)_{\node, \lab}\) and \(\classifier(A, X)_{\node, \lab'}\) of the original label $\lab$ and any other labels $\lab'$ for the node \(i\) being certified within the perturbation space (see \autoref{equ:node_robustness_min}). To (1)~approximate such difference, we use polyhedra \absint{} for the abstract domain to (2)~derive the lower bound for  \(\delta_{i, c, c'}\), which, in turn, will (3)~determine the GCN robustness for the node. 

\subsubsection{Robustness certification}
\emph{Step 1:}
To get bounds for \(\delta_{i, c, c'}\) in the abstract domain, we apply the \emph{fully connected} operation $\linear{\Delta^{\lab, \lab'}, 0}^{\#_P}$ to the output of the classifier \(\classifier\). Here $\Delta^{\lab,\lab'}$ is a $1\times|\labelSet|$ matrix where column $\lab$ is set to 1, column $\lab'$ is set to $-1$, and  all other entries are set to $0$. This corresponds to subtracting the score \(\classifier(A, X)_{\node, \lab'}\) from the score \(\classifier(A, X)_{\node, \lab}\) for each node of the graph. 

\noindent
\emph{Step 2:}
Certifying node level robustness requires to find the minimum value of \(\delta_{i, c, c'}\) allowed by the abstraction.
Formally, given an abstract element with coefficients  $\coef^{\prime}_{\node, c'}= \linear{\Delta^{\lab, \lab'},0}^{\#_P}(\coef_{\node})$ (for the original label $c$ and some other label $c' \ne c$) and the set of variables $\map'_\node$,  we 
define a \emph{minimization problem}:

\begin{equation}
    \min_{\map'_i} \delta_{i, c, c'} \textit{ subject to } \coef^{\prime}_{\node, c'} \cup \perturba{X}^{\localp, \globalp}
\end{equation}

Since inequalities in \(\coef^{\prime}_{\node, c'}\) are linear, the solution $\delta^*_{i, c, c'}$ of the minimization problem can be calculated by greedily flipping the feature variables with maximum negative changes as follows. 



\begin{enumerate}[label=\roman*)]
    \item For a $\layer$-layered GCN, let $\map{}'_\node{}$ denote the variables representing features from $\layer{}$-hop neighbors $\neighbors{\node{}}{\layer{}}$ for node $\node{}$: $\map'_\node = \bigcup_{k\in \neighbors{\node{}}{\layer{}}}\{\texttt{x}_{k,j}|j<m_0\}$. 
    Let $\polyLowExp{* i}$ be the reshaped matrix of
    $\polyLowExp{\prime i}$ from $\coef^{\prime}_{\node, c'}$ such that $(\polyLowExp{* i})_{k, j}$ corresponds to the coefficient of variable $\variable_{k, j} \in \map'_\node$ in the lower bound for each feature $j$ of node $k$.
    \item Using the input feature matrix $X$, the lower bound of score difference between label $c$ and $c'$ can be calculated as $\delta^{X}_{i, c, c'} = \sum_{\variable_{k,j}\in\map'_{i}}\bigl((\polyLowExp{*i})_{k, j} \cdot \nodeF_{k, j}\bigr) + \polyLowConst{i}$.
    \item Starting from $\delta^{X}_{i, c, c'}$, the minimum  $\delta^{*}_{i, c, c'}$ can be found by flipping the values of feature variable $\variable_{k, j}$ that reduce $\delta^{X}_{i, c, c'}$ the most within a local and global perturbation limit. 
    \item Let perturbation matrix $P$ be: $P_{k, \feature} = 1$ if $X_{k, \feature}=0$ and $P_{k, \feature} = -1$ if $X_{k, \feature}=1$. 
    Then $\theta_{k, j} = (\polyLowExp{*i})_{k, j} \cdot P_{k, \feature}$ represents the change of $\delta^{X}_{i, c, c'}$ if the value of variable $\variable_{k, j}$ is flipped. 
    \item Given a local perturbation limit $\localp$, we identify $\localp$ number of feature variables with the most negative changes for each node $k \in \neighbors{\node{}}{\layer{}}$. 
    Formally, these negative changes can be calculated as: $\hat{\theta} = \bigl[min\bigl(\op{MinK}_{\localp}([\theta_{k, j}]_{\variable_{k,j} \in \nodeVar_k}), 0\bigr)\bigr]_{k \in \neighbors{\node}{\layer}}$ where $\op{MinK_{k}}$ gets the minimum $k$ elements from a matrix.
    \item Given a global perturbation limit $\globalp$, we also identify the $\globalp$ number of most negative changes in $\hat{\delta}$. The sum of the $\globalp$ changes is the amount to deduct from $\delta^{X}_{i, c, c'}$ to derive the solution: $\delta^*_{i, c, c'} = \sum{\op{MinK}_{\globalp}(\hat{\theta}) + \delta^{X}_{i, c, c'}}.$
\end{enumerate}

\emph{Step 3:} If the lower bound of the minimum difference between the original label and other labels is larger than zero, then we can certify that node.
This \textit{lower bound} of the \textit{minimum difference} between the original label $c_i$ and any other label $c' \ne c_i$ of node $\node$ can be defined as  $r^*_i = min([\delta^*_{\node, \lab_i, \lab'}]_{\lab' \ne \lab_i})$. Using polyhedra \absint{}, the \emph{certification judgement for a node} $i$ is defined as $\certifier_\le(\graph, \perturba{\nodeF}, \classifier)_i = r^*_i$, and a node is  certified if $r^*_i>0$.
The GCN certifier is composed of node-level judgements $\certifier_\le(\graph, \perturba{\nodeF}, \classifier) = [r^*_i]_{i \le n}$. 

\begin{theorem} (Soundness of robustness certification)\\
\label{thm:poly_sound}
Certifier $\certifier{}_{\le}: (\adj{}, \perturba{\nodeF{}}, \classifier{}) \to \realDomain{}^\numNodes{}$ is sound; thus it produces a lower bound for the robustness of $\classifier$ over $\graph$.
\end{theorem}
\chadded{Proof for \autoref{thm:poly_sound} is demonstrated in Appendix \ref{append:proof_poly_sound}.}



\subsubsection{Counterexamples of robustness}
If a node $i$ is not certified, 
the steps above also enable the generation of counterexamples that provide an upper bound for the GCN's robustness. 

Let $\set{X}^{min}_{i,c'}$ be the set of variables selected for calculating $\delta^*_{i, c, c'}$, and $\nodeF^{\node c'}$ be the perturbed input feature matrix by flipping values of the entries corresponding to variables in $\set{X}^{min}_{i,c'}$ such that each entry $\nodeF^{\node c'}_{k, j}$ ($ \variable_{k, j} \in \nodeVar$) is defined as 
\noindent{\small\setlength{\abovedisplayskip}{0pt}\setlength{\belowdisplayskip}{0pt}
\begin{equation*}
    \nodeF^{\node c'}_{k, j} = \begin{cases}
        1 - \nodeF_{k, j} &\text{if } \variable_{k, j} \in \set{X}^{min}_{i,c'}\\
        \nodeF_{k, j} &\text{otherwise.}
    \end{cases}
\end{equation*}
}
$\nodeF^{\node\lab'}$ is a \emph{counterexample for robustness} of GCN $\classifier$ for node $\node$ if $argmax(\classifier(\adj,\nodeF)_\node) \ne argmax(\classifier(\adj,\nodeF^{\node\lab'})_\node)$. Let $\certifier_{\ge}$ be a certifier using polyhedra \absint{} with $\certifier_{\ge}(\graph, \perturba{\nodeF}, \classifier):=r^{*\prime}$ such that $r^{*\prime}_i = 0$ if a counterexample is found for node $i$ and $r^{*\prime}_i=1$ otherwise.
\begin{theorem} (Completeness of counterexamples) \\
\label{thm:poly_compelte}
Certifier $\certifier{}_{\ge}: (\adj{}, \perturba{\nodeF{}}, \classifier{}) \to \realDomain{}^\numNodes{}$ is complete and provides an upper bound for graph-level robustness.
\end{theorem}
\chadded{Proof for \autoref{thm:poly_compelte} is shown in Appendix \ref{append:proof_poly_complete}.}


Thanks to the over-approximation property of polyhedra \absint{}, non-robust certification only needs to be performed on nodes with a negative lower-bound difference: $r^*_\node < 0$, since a node cannot be certifiably robust and non-robust at the same time.


\begin{figure}
    \centering
    \includegraphics[width=\linewidth]{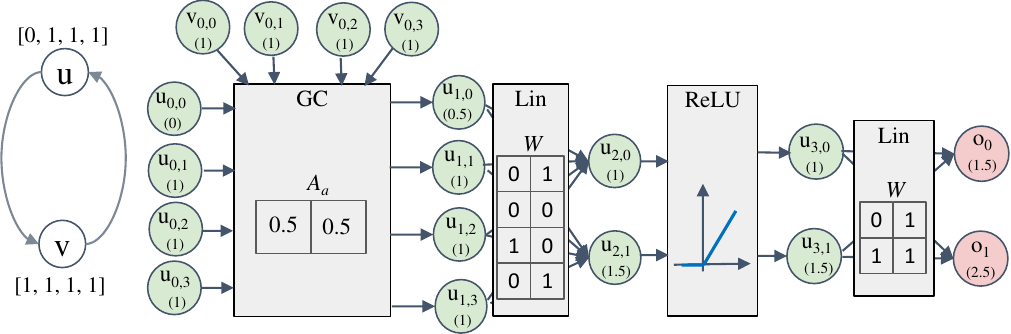}
    \caption{A sample graph and a GCN node classifier. Numbers in each component represent the weights matrix. Given a graph with two nodes $u$, $v$ and input features next to it, the GCN computes a score of each label for node $u$. Intermediate values (represented as green circles) at each layer $l$ and for dimension $k$ are denoted as $u_{\text{l}, \text{k}}$, while output values for a label $c$ are represented by $o_{\text{c}}$. The GCN predicts label $1$ for node $u$ since $o_1 - o_0> 0 $. If one can perturb exactly one node value, interval \absint{} is not able to certify the robustness of this example due to imprecision, while polyhedra \absint{} can successfully certify the robustness.}
    \label{fig:example}
\end{figure}

\subsubsection{Example certification}
\label{example:poly}
Figure \ref{fig:example} shows a simple graph with two nodes and two edges forming a loop and a simple GCN with one graph convolution layer and two linear layers classifying the nodes into one of the two labels (0 and 1). \chadded{The numbers in the graph convolution and linear components correspond to the normalized adjacency and weight matrices, respectively.} For simplicity, only \chadded{parts of the calculation} focusing on node $u$ is shown. 

\chadded{Using the input graph structure and node features, the GCN computes $o_0 = 1.5$ and $o_1 = 2.5$ for node $u$.} Thus, the node $u$ will be classified to label $1$ on the original graph \chadded{since $o_1 - o_0 > 0$}. Next, we aim to certify if the GCN is robust for node $u$ such that the GCN will always classify $u$ to label $1$ when we can flip exactly one node feature ($\localp=\globalp=1$).

\noindent
\textbf{Interval \absint{}.}
\chadded{
Interval \absint{} \cite{liu2020abstract} first estimates the interval bounds \emph{after the first GCN layer} (graph convolution and linear) using symbolic bounds, and then propagates them through interval arithmetic. Further details of the interval \absint{} approach are provided in Appendix \ref{append:interval}.
}

\chadded{
In this example, the first two layers are computed symbolically. After the $\gc$ layer, the latent variables can be expressed as $u_{1,j} = 0.5(u_{0,j} + v_{0,j})$ for $j \in [0,3]$, while the linear layer yields $u_{2,0} = u_{1,2}$ and $u_{2,1} = u_{1,0} + u_{1,3}$. Substituting $u_{1,j}$ with the expression from the $\gc$ layer gives $u_{2,0} = 0.5(u_{0,2} + v_{0,2})$ and $u_{2,1} = 0.5(u_{0,0} + v_{0,0} + u_{0,3} + v_{0,3})$. For $u_{2,0}$, since only one feature can be perturbed and $u_{0,2} = v_{0,2} = 1$ in the original node features, the lower bound is $0.5$ (by flipping one feature to $0$) and the upper bound is $1$ (no perturbation). Similarly, we obtain $u_{2,1} \in [1,2]$.
}

\chadded{
The remaining calculations are carried out using linear arithmetic. Since the lower bounds of $u_{2,0}$ and $u_{2,1}$ are both positive, the $\relu{}$ function acts as the identity function, yielding $u_{3,0} \in [0.5,1]$ and $u_{3,1} \in [1,2]$. The output bounds are then given by
\begin{align*}
&o_0 = 0 \times u_{3,0} + 1 \times u_{3,1} = 0 \times [0.5,1] + 1 \times [1,2] = [1,2], \\
&o_1 = 1 \times u_{3,0} + 1 \times u_{3,1} = 1 \times [0.5,1] + 1 \times [1,2] = [1.5,3], \\
&o_1 - o_0 = [1.5,3] - [1,2] = [-0.5,2].
\end{align*}
Since the lower bound of the difference between two labels is $-0.5 < 0$, the robustness of node $u$ cannot be certified.
}

\noindent
\textbf{Polyhedra \absint{}.}
Now, we perform the same certification using the polyhedra \absint{} proposed in this paper.
From \autoref{sec:abst-oper}, one can express each latent feature as linear inequalities of input terms. For simplicity, we ignore the bias term in the abstract element and use equal sign when the upper and lower bounds are equal. 

The calculation for the initial two layers are similar to the interval \absint{}.
\chadded{Since the first $\gc$ layer is linear, the bounds can be calculated precisely}, $u_{1,j} = 0.5 (u_{0,j} + v_{0,j}), j=[0, 3]$. The linear layer also gives exact bounds $u_{2,0}=0.5(u_{0,2} + v_{0,2})$ and 
$u_{2,1}=0.5(u_{0,0} + v_{0,0} + u_{0,3} + v_{0,3})$. 

\chadded{Different from the interval approach, polyhedra \absint{} continue calculates the symbolic bounds using the interval bounds.}
From the interval \absint{} described previously, $u_{2,0}\in [0.5, 1]$ and $u_{2,1} \in [1, 2]$, since both lower bounds are positive, $\relu$ becomes an identity function: $u_{3,0} = u_{2, 0}$ and $u_{3, 1} = u_{2, 1}$. Finally, the output bounds can be expresed as:
\begin{align*}
    &o_0 = 0 \times u_{3, 0} + 1\times u_{3, 0} =0.5(u_{0,0} + v_{0,0} + u_{0,3} + v_{0,3}) \\
    &o_1 = 1 \times u_{3, 0} + 1\times u_{3, 0} =0.5(u_{0,0} + v_{0,0} +  u_{0,2} + v_{0,2} + u_{0,3} + v_{0,3}) \\
    &o_1-o_0=0.5(u_{0,2} + v_{0,2})
\end{align*}
In the original graph, both $u_{0,2}$ and $v_{0,2}$ are $1$. Since we can only perturb one node feature, the lower bound of $o_1 - o_0 = 0.5 > 0$. Thus, we successfully certified the robustness of node $u$ for the GCN. \chadded{Notably, since the $\relu{}$ function acts as the identity function in this case, polyhedra \absint{} captures the GCN’s behavior \emph{exactly}, i.e., the lower and upper bounds are the same.}

\subsubsection{Implementation:}
The abstract operations are \textit{reversible}. One can start from the last layer and continuously back-substitute the latent feature variables until the input feature variables are obtained. This approach can further reduce memory and computation needs for certification by reducing unnecessary neighboring node computations, making the certification more efficient. Our implementation uses differentiable matrix operations to exploit the benefits of GPUs. Further details for implementation and experiments can be found in our reproducibility package.\footnote{\url{https://github.com/20001LastOrder/abstract_interpretation_for_gcn}}

\subsection{Measuring tightness of a certification approach}
\label{sec:uncertainty_region}



To provide a \textit{holistic} view of the certification approach's performance, we propose a new evaluation metric for a pair of robustness certifiers that provide lower and upper bounds by generalizing existing metrics \cite{zugner2019certifiable,singh2019abstract}. The \emph{uncertainty region} is the area between the upper and lower bounds of a perturbation limit range. For a specific perturbation limit, the difference between the upper and lower bound gives the percentage of nodes where the approach is unsure about its robustness. Then, the uncertainty region represents the total uncertainty of the certifier for a limit range. 

Given a pair of certifiers 
that produces a lower and upper bound for robustness with local limit $\localp$ and global limit $\globalp$, let $[\localp_\le, \localp_\ge]$ be the local perturbation range and $[\globalp_\le, \globalp_\ge]$ be the global perturbation range, we define the uncertainty region of the pair of certifier in perturbation range as
\begin{equation}
\text{\small$
    \label{equ:uncertainty_region}
    \int^{\globalp_\ge}_{\globalp_\le}\!\int^{\localp_\ge}_{\localp_\le}\!{\graphRobustUp{\localp}{\globalp} - \graphRobustLower{\localp}{\globalp}}\:d\localp\:d\globalp$
}
\end{equation}

 \begin{figure}[t]
     \centering \includegraphics[width=0.75\linewidth]{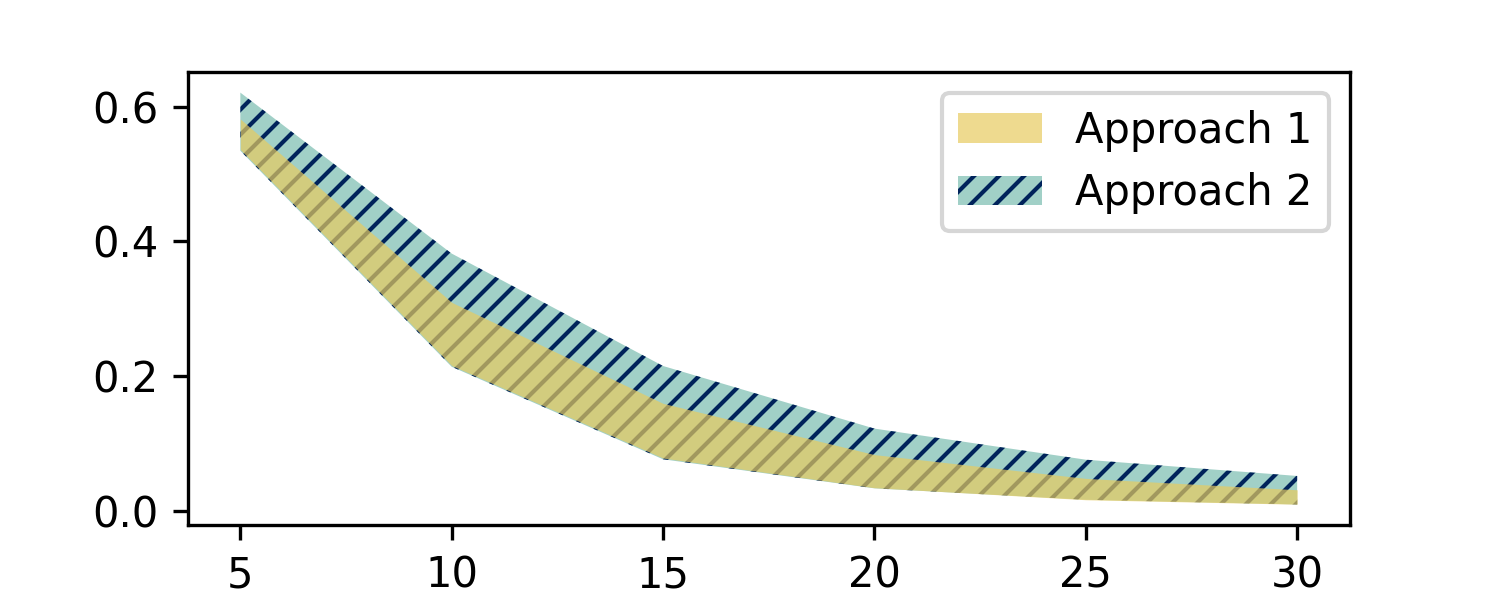}
     \caption{Uncertainty regions (colored areas)}
     \label{fig:region_example}
 \end{figure}

 \begin{example}
 \autoref{fig:region_example} exemplifies the uncertainty region of two certification approaches. In the example, the local perturbation is fixed, and the global perturbation ranges from 5 to 30. Borders of the shapes represent the upper and lower bound given by the two approaches; the colored areas are the uncertainty regions. In this example, approach 1 (yellow area) is more precise than approach 2 (blue shaded area) as it has a smaller uncertainty region. 
 \end{example}

 \subsection{Collective certification}
 \label{sec:collective-certification}
 Our method can also be used to derive the input for a collective certifier, enabling collective certification for the entire graph. The collective certifier evaluates the graph as a unified input for all nodes and derives the \textit{lower bound} of the collective robustness. To calculate the collective certification from single-node certificates, the certifier uses the highest (global) perturbation limit for each node while its label remains unchanged, i.e., it is still a robust node in the context of single-node certificate. We designate this value as the \textit{maximum robust limit} for each node. Formally, the maximum robust limit $\hat{p}^i$ for a node $i$ with a predetermined local perturbation limit satisfies the following conditions:
 \begin{equation}
     \hat{p}^i: \nodeRCP{i}{p_l}{\hat{p}^i}_\le = 1 \land \forall p_g > \hat{p}^i, \nodeRCP{i}{p_l}{p_g}_\le = 0
 \end{equation}
 
 Following this, the collective certifier determines the \textit{greatest} perturbation across the graph such that for each node, the number of perturbations affecting it remains below its maximum robust limit, while respecting the local perturbation limit. \autoref{alg:collective} illustrates the process of computing the maximum robust limit with a classifier. Starting with the original feature matrix, the perturbation limit is incrementally increased until the certifier fails to certify the node. Then the highest limit where the classifier flags the node as robust is the maximum robust limit.  

\begin{algorithm}
\caption{Pseudocode to find the maximum robust limit for a node}
\label{alg:collective}
\raggedright
\textbf{Input:} node $i$, local budget $\localp$, certifier $\certifier$, GCN classifier $\classifier$, and input graph $\graph=(\adj, \nodeF)$ \\
\KwResult{maximum robust limit $\hat{p}^i$ for the node $i$}
$p \gets 0$\;\\
\While{$\nodeRCP{i}{p_l}{\hat{p}^i}_\le = 1$}{
  $p \gets p + 1$\;
}
\Return p - 1
\end{algorithm}


\subsection{Robust training}
\label{sec:robust-training}
Since the certification process is differentiable, it can also be used for robust training. We define the loss function used in the robust training for each node. The final loss function for a batch can be the average of all loss values of the nodes in the batch. We also omit training details, such as backpropagation and gradient update.

Given a graph $\graph=(\adj,\nodeF)$ and perturbation space $\perturba{\nodeF}$, let classifier $\classifier$ be a GCN, and $\classifier^{\#_P}$ be the polyhedra abstract classifier corresponding to $\classifier$. For the input abstraction $\mathbf{a}_0$ of the perturbation space, let $\absele_\layer = \classifier^{\#_P}(\mathbf{a}_0)$ be the output abstract element. For any node $\node$, we compute $\delta^*_{\node, \lab_\node} = [\delta^*_{\node, \lab_i, \lab'}]_{\lab' \ne \lab_i}$ for some label $\lab_\node$ as the lower bound between the output score differences of $\lab_\node$ and other labels.


During training, we can use the ground  truth label if the node is labeled. Let $\lab_{gt}$ be the ground truth label of node $\node$: $\lab_\node = \lab_{gt}$. When a node $\node$ is not labeled, we can either ignore it during the training or use the label predicted by the GCN $\lab_\node = argmax(\classifier(\adj, \nodeF)_\node).$

Ideally, $\delta^*_{\node, \lab_{\node}}$ should be a matrix of positive values, meaning $\classifier$ predicts the ground truth label for all possible changes in the perturbation region. To achieve this, the binary cross entropy loss function can be used during training: $\op{bce}(\delta^*_{\node, \lab_{\node}}) = \mathbf{1}^{(1, |\labelSet| - 1)} \cdot \log \sigma(\delta^*_{\node, \lab_{\node}})$, where $\sigma$ is the sigmoid function and $\labelSet$ is the set of all labels.

 In some cases, controlling the values in $\delta^*_{\node, \lab}$ is necessary so that it does not grow indefinitely. With this in mind, one can use the robust hinge loss \cite{zugner2019certifiable}: $\op{hinge}_{t}(\delta^*_{\node, \lab})=max(-\delta^*_{\node, \lab} + t, 0)$ where $t$ is some predefined threshold. The hinge loss only penalizes the GCN if $\delta^*_{\node, \lab}$ is smaller than $t$. Consequently, this loss is suitable for semi-supervised learning with unlabeled nodes. One can give a larger threshold for labeled nodes with ground truth labels and a smaller threshold for unlabeled nodes with predicted labels.

\subsection{Discussion}
This section presents our GCN certification approach based on polyhedra \absint{} in detail. Despite its increased computational demands, the polyhedra-based approach offers many advantages over the interval-based approach.

\textbf{Abstraction tightness:} The polyhedra abstract domain provides a notably tighter approximation to the perturbation space when compared to the interval domain. The interval domain approximates possible outputs of the GCN with a multi-dimensional box, as it solely tracks the minimum and maximum values. In contrast, the polyhedra domain approximates the values through linear inequalities of input features, resulting in a more precise approximation. Indeed, as demonstrated in \autoref{example:poly}, the interval-based method fails to certify the GCN under a simple example, whereas the polyhedra-based approach can readily certify the example.

\textbf{Counterexamples:} In addition to the standard lower-bound robustness offered by the interval-based method, the polyhedra-based approach can also create potential counterexamples, identifying upper bounds of robustness. In the polyhedra domain, the coefficient for each input feature used to approximate the output of a GCN can be treated as the impact of altering that particular input feature on the output. Thus, the set comprising the influential features could form a counterexample. Such counterexamples serve as valuable input for understanding the failure patterns of a GCN.

\textbf{Applicability:} Our method can serve as input for other techniques aimed at deriving other types of robustness certification or enhancing robustness. While the interval-based method may also be used for collective certification, its inherent imprecision can significantly impact the resulting performance when compared to the tight approximation of the polyhedra domain.  Moreover, our method can further improve the robustness of GCNs by integrating lower bounds into the loss function during robust training, facilitated by the differentiable operations used in the method. 

\section{Evaluation}
\label{sec:evaluation}


In order to evaluate the effectiveness and performance of our \absint{}-based robustness certification approach, we conduct experiments on three well-known node classification datasets. We target the following research questions in our experiments and perform all experiments on a server with a single Nvidia RTX A6000 GPU:

\begin{enumerate}[label=RQ\arabic*,noitemsep,topsep=0pt, leftmargin=*]
    \item How tight is the certification based on polyhedra \absint{}?
    \item How is the runtime performance of polyhedra Abs\-Int certification compared to existing approaches?
    \item How effective is the approach to certify \textit{collective robustness} of the entire graph?
    \item What is the effect of the polyhedra \absint{} certification on robust training?
\end{enumerate}

\subsection{Experiment setup}
\noindent\textbf{Datasets:} We evaluate our approach using three citation graphs often used for benchmarking node classification tasks: Citeseer \cite{giles1998citeseer}, 
Cora \cite{mccallum2000automating}, 
and PubMed 
\cite{sen2008collective} (see detailed statistics of the datasets in \autoref{tab:uncertainty_region}). 
Although graph sizes in these datasets are relatively small compared to large-scale graphs such as social networks and knowledge graphs, they are widely recognized as benchmarks for previous work of GNN certifications\cite{lee2019tight,liu2020abstract,zugner2019certifiable}. We designate the assessment of our approach on large-scale graphs to future work.  
To set up a baseline, we follow existing configurations \cite{zugner2019certifiable} for each dataset. The fixed local perturbation budget is set to $1\%$ of the number of nodes features $\localp = 0.01\numFeatures$, while the global perturbation range is set to between $1$ and $50$: $\globalp \in[1, 50]$.


\noindent\textbf{Baseline approaches:}
We compare our approach with two prevalent node classifier certification techniques. The first achieves the current state-of-the-art certification tightness, while the second also relies on \absint{} as our method. Although other certification methods exist, they typically follow a black-box approach, leading to probabilistic robustness certifications on a modified version of the GNN, or they focus on tasks other than node classification. Consequently, these methods fall outside the scope of this paper. 

The \emph{dual optimization approach} (\textbf{\Dual{}}) \cite{zugner2019certifiable} uses dual programs to provide a lower and upper bound on the robustness of a GCN over a given graph. It achieves the current state-of-the-art performance in certifying node feature perturbations for node classification.
\chadded{During our experiments, we identified an issue in the original implementation, where accumulated numerical errors from sparse matrix multiplications led to imprecise upper bounds. We resolved this issue by consistently using dense matrix multiplications, which yielded tighter upper bounds. We refer to the original implementation as \textbf{\DualO{}} and the corrected version as \textbf{\DualF{}}. 
}

The interval \absint{} \cite{liu2020abstract} uses 
interval as the abstract domain (\textbf{\Interval{}}) and ignores the relations between node features. We implemented this method as our numerical activation bounding. The intervals provide fast estimations of the lower bound robustness. However, this approach cannot produce any counterexamples. The paper mentions two variations (Max and TopK) of the method; we only use the more precise version (TopK) in the experiments.

\noindent\textbf{Our approaches:}
We evaluate two variations of our technique: when performing activation bounding with the interval abstraction before polyhedra abstract interpretation, one can choose \textit{Max} perturbation (\textbf{\OursM{}}) or \textit{TopK} perturbations (\textbf{\OursT{}}). \textit{Max} yields a faster certification but \textit{TopK} is more precise. 

\noindent\textbf{Compared metrics:}
To compare the tightness of certification (RQ1), we evaluate (1) the lower bound of robustness for different perturbation budgets, and (2) the uncertainty region over the perturbation range. Since the local perturbation budget is fixed to $\localp$ and the global perturbation budget is discrete, we can simplify \autoref{equ:uncertainty_region} of uncertainty region to one summation:
\begin{equation}
\label{equ:ur_evaluation}
\text{\small
    $\textstyle\sum^{\globalp_\ge}_{\globalp=\globalp_\le}{(\graphRobustUp{\localp}{\globalp} - \graphRobustLower{\localp}{\globalp})}$
}
\end{equation}
Runtime performance (RQ2) is evaluated using the time needed to finish the certification with different perturbation budgets. Then, the collective certification is measured by the lower bound collective robustness (RQ3). Finally, robust training (RQ4) is evaluated using accuracy and the lower bound robustness.



\begin{table}[tb]
    \centering
    \begin{tabular}{|c|c|c|c|}
    \hline
     & Citeseer & Cora & Pubmed \\
    \hline
    Nodes & 2995 & 3312 & 19171 \\
    Edges & 8416 & 4715 & 44324 \\
    Features & 2876 & 4715 & 500 \\
    Labels & 7 & 6 & 3\\
     \hhline{|====|}
    \DualO{} &3.32  & 6.24 & 9.60 \\
    \DualF{} &1.81 & 4.51  &  \textbf{5.81}  \\ \hline
    \OursT{} &\textbf{1.79} & \textbf{4.42}  & 5.93 \\
    \OursM{} &2.24 & 6.50  & 6.65\\
    \hline
    \end{tabular}
    \caption{Dataset statistics and uncertainty region for certifiers (lower is better)}
    \label{tab:uncertainty_region}
\end{table}

\subsection{RQ1: Certification tightness}
\noindent\textbf{Setup:}
We perform two different analyses to evaluate the tightness of the certification methods. First, we check lower bounds from each method. Since all methods are sound, a higher lower bound means the result is closer to the actual robustness, thus, more precise.

Next, we evaluate the uncertainty region for all methods when 
both the lower bounds and upper bounds are evaluated over a range of perturbation budgets. A lower uncertainty region means the certification method is tighter over the \textit{entire perturbation range}. This evaluation excludes \textit{\Interval{}} as it does not produce an upper bound. We use pre-trained GCNs in this experiment \cite{zugner2019certifiable}.

\noindent\textbf{Analysis:}
\autoref{tab:lowerbound} shows lower bounds produced by each method over four different perturbation budgets. We merge the result of \textit{\DualO{}} and \textit{\DualF{}} as they give the same lower bound. In general, \textit{\Interval{}} is the least tight method, giving each perturbation's lowest value on all datasets. On the contrary, \textit{\OursT{}} gives the most precise lower bound for all cases. This can be attributed to the optimum bounding area used in bounding the non-linear activations. The \textit{\Dual{}} method also gives precise lower bounds that are slightly lower than \textit{\OursT{}}. Interestingly, the sets of nodes certified by \textit{\Dual{}} and \textit{\OursT{}} are not exactly the same. This difference means one can run both \textit{\Dual{}} and \textit{\OursT{}}, take the union of the verified nodes, and get an even higher lower bound on the robustness.

The uncertainty region values of the methods computed over the perturbation range 0 and 50 are shown in \autoref{tab:uncertainty_region}. One can observe that the fixed version \textit{\DualF{}} significantly improves the precision of the upper bound compared to \textit{\DualO{}}. The two approaches have the same lower bounds, but \textit{\DualF{}} has a much smaller uncertainty region compared to \textit{\DualO{}} across all datasets. Meanwhile, \textit{\OursT{}} is the tightest approach in 2 out of 3 datasets while being slightly worse than \textit{\DualF{}} in Pubmed. This result implies that \textit{\OursT{}} can also provide a comparable upper bound as \textit{\DualF{}}, thanks to the tight linear bounds provided by the polyhedra operators.


\begin{table*}[t!]
\centering
\scalebox{0.9}{
{\footnotesize\begin{tabular}{|c|cccc||cccc||cccc|}
\hline
\multicolumn{1}{|c|}{\multirow{2}{*}{Approach}} & \multicolumn{4}{c||}{Citeseer} & \multicolumn{4}{c||}{Cora} & \multicolumn{4}{c|}{Pubmed} \\
\hhline{~------------}
& 1 & 10 & 20 & 30 & 1 & 10 & 20 & 30 & 1 & 10 & 20 & 30 \\
\hline
\Interval{} & 75.75 & 2.14 & 0.03 & 0.0 & 86.38 & 13.72 & 1.27 & 0.13 & 73.65 & 2.23 & 0.06 & 0.01\\
\Dual{} & 86.9 & 21.38 & 3.32 & 0.94 & \textbf{93.69} & 55.03 & 26.61 & 11.22 & 84.97 & 12.94 & 4.22 & 1.66 \\ \hline
\OursT{} & \textbf{86.96} & \textbf{21.56} & \textbf{3.41} & \textbf{0.97} & \textbf{93.69} & \textbf{55.16} & \textbf{27.15} & \textbf{11.45} & \textbf{85.32} & \textbf{13.6} & \textbf{4.7} & \textbf{1.97} \\
\OursM{} & \textbf{86.96} & 18.72 & 2.14 & 0.63 &\textbf{93.69} & 51.39 & 19.6 & 5.84 & \textbf{85.32} & 9.24 & 1.85 & 0.44 \\
\hline
\end{tabular}}}
\caption{Certifier lower bounds for different perturbation budgets, higher values mean tighter bounds (all values are in $\%$)}
\label{tab:lowerbound}
\end{table*}

\begin{tcolorbox}
\textbf{Answer to RQ1.}  
The polyhedra abstract interpretation-based certification method \OursT{} provides tighter lower bounds on the robustness compared with all other baseline methods. Moreover, within a perturbation range of 50, its uncertain region is smaller than that of all other baselines in two out of the three datasets, while its performance is on par with \DualF{} when evaluated on the Pubmed dataset.  
\end{tcolorbox}

\subsection{RQ2: Runtime performance}

\noindent\textbf{Setup:} 
In this evaluation, we perform the same certification as in the previous section and measure the runtime for each approach. For approaches that provide both lower and upper bounds, we run the certification to get both bounds. For the \textit{\Interval{}} method, we only get the lower bound. We use the median of ten different runs.

\begin{figure*}[tb]
    \centering
    \includegraphics[width=\textwidth]{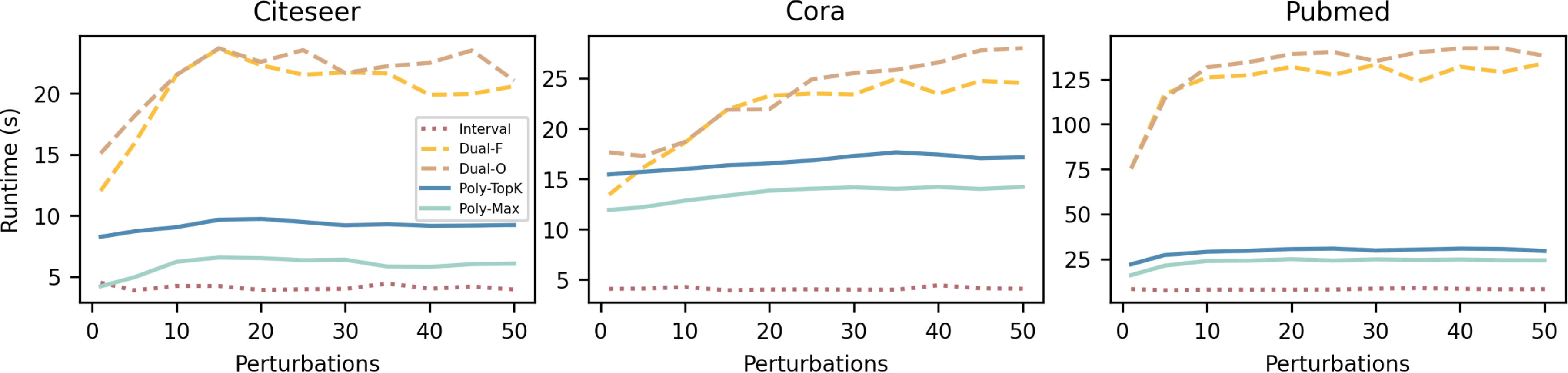}
    \caption{Runtime for certifiers (in seconds)}
\label{fig:runtime}
\end{figure*}

\noindent\textbf{Analysis:} \autoref{fig:runtime} shows the runtime evolving for all three datasets over the perturbation range. In general, the runtime of all approaches increases as the perturbation budget increases. The PubMed dataset has the longest runtime since it has the most number of nodes to be certified. \textit{\DualF{}} and \textit{\DualO{}} have similar runtime, hence our fix to the upper bound does not introduce any performance overhead. Meanwhile, \textit{\OursM{}} is faster than \textit{\OursT{}}, showing the tightness-runtime tradeoff between the two settings.

\textit{\Interval{}} has the shortest runtime for all three datasets. However, it also provides the worst precision in terms of lower-bound robustness. The figure also shows that our approach runs faster than the \textit{\Dual{}} method while still providing a better certification tightness for Citeseer and Cora. Remarkably, our method also scales better with respect to the perturbation budget and size of the dataset. It almost stays flat with respect to the perturbation budget while the \textit{Dual{}} approach has a significant jump at around 10 perturbation budget. At the same time, when moving from a smaller dataset Citeseer to a larger dataset Pubmed, the runtime for \textit{\Dual{}} increases by more than 6 times, while the runtime for our approach only increases by around 3 times. {This significant improvement in runtime can be attributed to our integration of interval \absint{} for activation bounding and the efficient GPU-based parallel implementation.

\begin{tcolorbox}
\textbf{Answer to RQ2.}  
Both abstract interpretation approaches \OursT{} and \OursM{} exhibit significantly better runtime performance compared to the \Dual{} approaches, while simultaneously providing tighter certification. Additionally, their runtime escalation with respect to the perturbation budget is more gradual compared to the \Dual{} approaches.
\end{tcolorbox}

\subsection{RQ3: Collective robustness}
\noindent\textbf{Setup:}
In this research question, we investigate the impact of single-node certification obtained through our method \textit{\Ours{}} (\OursT{}) on collective certification in comparison to the \textit{\Dual{}} approach when incorporated as input to a \textit{collective certifier}. We ignore the \Interval{} approach as it provides significantly less precise lower bound robustness compared to the other two approaches. For this experiment, we use the current state-of-the-art method for collective certification based on linear optimization proposed by Schuchardt et al. \cite{schuchardt2023collective}. For both \Dual{} and \Ours{} methods, we derive inputs for the collective certifier by determining the \textit{maximum} robust limit for each node following the procedure in \autoref{alg:collective}. 
We ignore the local budget and only focus on the global budget in the experiment for simplicity. 

The collective certifier categorizes perturbation based on \emph{feature addition} (changing a feature from 0 to 1) and \emph{feature deletion} (changing from 1 to 0). Given that \textit{\Ours{}} calculates a lower bound expression using node feature variables, it enables fixing the value of some feature variables during robustness certification to accommodate both perturbation types. Conversely, the \textit{\Dual{}} approach does not distinguish the two perturbation types. Thus, we set the same value for feature addition and feature deletion to the collective certifier for \Dual{}. 

\begin{figure}[tb]
    \centering
    \subfloat{\includegraphics[width=0.45\linewidth]{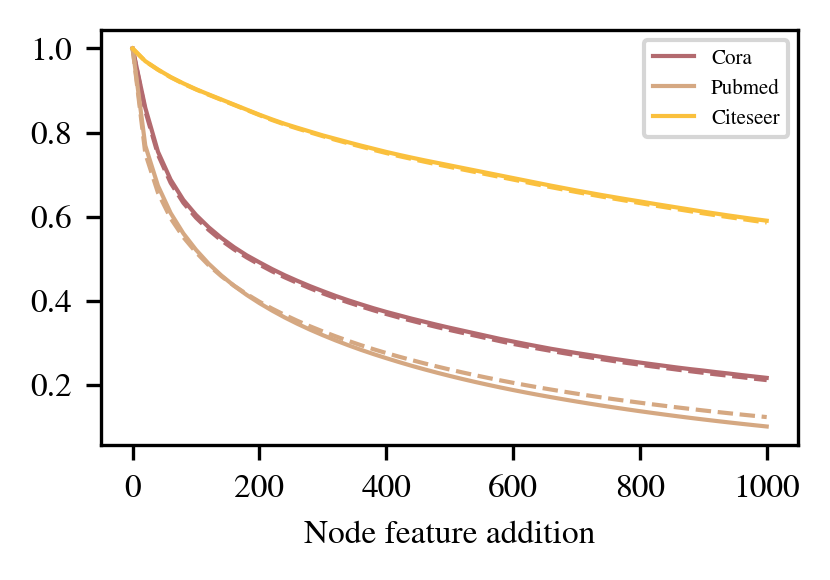}}
    \quad
    \subfloat{\includegraphics[width=0.45\linewidth]{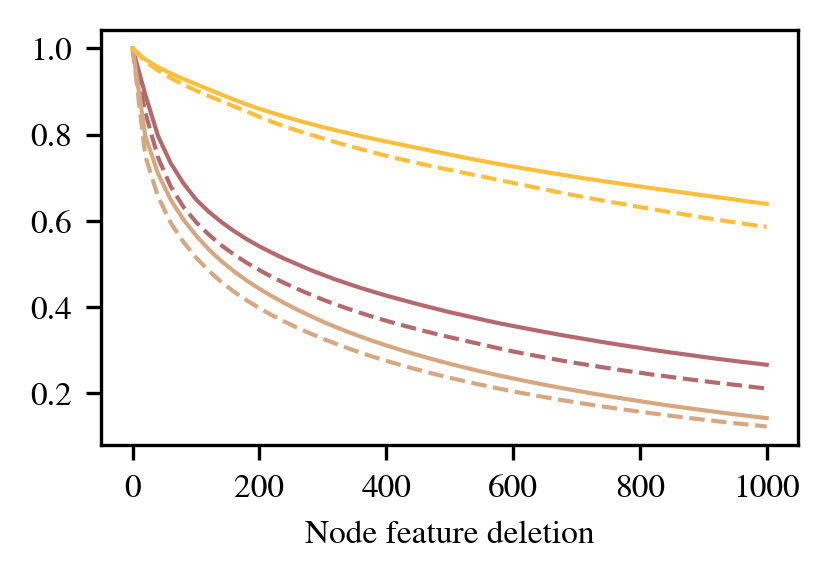}}
    \caption{Collective certification results (Higher means tighter certification). We compare \textit{\Ours{}} (solid line) and \textit{\Dual{}} (dashed line) approaches on feature addition (left) and deletion (right).}
    \label{fig:collective}
\end{figure}

\noindent \textbf{Analysis:}
\autoref{fig:collective} demonstrates the performance of the two approaches on collective certification of the three datasets. We note that the value of collective certification is significantly higher compared to the value of single-node certification. For example, both methods yield a single-node certification value of 0 for Citeseer when the perturbation budget exceeds 150. Yet, during collective certification, robust nodes still exist at a perturbation budget of 1000. Such difference aligns with the findings of prior work~\cite{schuchardt2023collective}, underlining the differences between the two attack models.

Generally, \textit{\Ours{}} outperforms the \Dual{} approaches on all three datasets in feature addition and deletion. The only exception is for feature addition in the Pubmed dataset, where the performance of \Ours{} is slightly lower than that of \Dual{}. We attribute the improvement of \Ours{} over \Dual{} to the tighter single-node certification lower bound derived by our approach, coupled with its capability to adapt to different perturbation types. Interestingly, the improvement observed in feature deletion is more significant than in feature addition. A possible explanation for this disparity is the sparsity of the feature matrix. For example, in both datasets, fewer than $2\%$ of the features have a value of 1, leading to a significantly larger number of feature addition combinations compared to feature deletion combinations. 

Interestingly, \autoref{tab:lowerbound} shows that the single node certification in Cora is generally higher than that in Citeseer. Conversely, the collective certification of Citeseer is much higher than Cora. We suspect this divergence is caused by the different graph structures, such as the density of the graph.

\begin{tcolorbox}
\textbf{Answer to RQ3.}  
When used as input for collective certification, \Ours{} mostly outperforms the \Dual{} approach. Notably, the improvement is more significant in the context of feature deletion than feature addition. 
\end{tcolorbox}


\subsection{RQ4: Robust training}
\begin{table}[t!]

\makebox[\linewidth][c]{
\scalebox{0.8}{
\begin{tabular}{|c|cc|cc|cc|cc|cc|cc|}
\hline
\multicolumn{1}{|c|}{\multirow{2}{*}{Method}} & \multicolumn{2}{c|}{Citeseer} & \multicolumn{2}{c|}{Citeseer-U} & \multicolumn{2}{c|}{Cora} & \multicolumn{2}{c|}{Cora-U} & \multicolumn{2}{c|}{PubMed} & \multicolumn{2}{c|}{PubMed-U}  \\
\hhline{~------------}
& \makecell{LB@12} & Acc & \makecell{LB@12} & Acc & \makecell{LB@12} & Acc & \makecell{LB@12} & Acc & \makecell{LB@12} & Acc & \makecell{LB@12} & Acc\\
\hline
Standard CE & 25.51 & 69.53 & -- & -- & 60.83 & 83.59 & -- & -- & 23.09 & 84.49 & -- & -- \\
\hhline{|=============|}
 \Dual{} + RH & \textbf{51.40} & 68.92 & 77.78 & \textbf{70.31} & 74.60 & 83.01 & \textbf{91.70} & 82.75 & 81.08 & 83.63 & 81.73 & \textbf{83.62} \\
\Ours + RH & 50.55 & 68.86 & 77.88 & 68.8 & 74.70 & 83.11 & 91.15 & 82.92 & 80.44 & 83.19 & 81.52 & 83.4 \\
\hhline{|=============|}
\Dual{} + BCE & \textbf{54.81} & 69.59 & 76.24 & 69.95 & 76.48 & 83.09 & 90.57 & 82.83 & 79.65 & 83.33 & 79.66 & 83.23 \\
\Ours + BCE & 53.53 & 69.30 & 75.68 & 69.83 & 76.54 & 83.33 & 90.08 & 82.8 & 79.88 & 83.46 & 79.8 & 83.38 \\
 \hline
\end{tabular}
}
}
\caption{Robustness lower bound and test accuracy of robust training results, average over 5 random seeds, higher is better}
\label{tab:robust_training}
\end{table}

\noindent\textbf{Setup:}
In this experiment, we aim to evaluate the effect of our certification method when used in robust training. 
We use two loss functions: the robust binary cross-entropy loss (\emph{\bce}) and the robust hinge loss (\emph{\hinge}) (see \autoref{sec:robust-training}). We exclude the regular cross-entropy loss since the robust loss already considers ground truth labels. In RQ1, we observe \emph{\OursT{}} provides the best lower bound. So we use it to evaluate the robustness of the trained models. 
 
We compare our approach with standard cross-entropy training and robust training with the \Dual{} method (\Interval{} does not provide robust training). We label the training using standard cross entropy as \textbf{CE} and robust training with further robust losses as \textbf{Method name + Loss name}. We only use the \textit{\OursM{}} approach in robust training. The numerical bound must be re-calculated for each training step since the weights are continuously updated. The \textit{\OursT{}} method would make the training significantly slower compared with the standard cross-entropy training. 

Finally, we also evaluate the training in a semi-supervised setting, where we provide the ground truth label in the robust loss for labeled nodes and the predicted label for unlabeled nodes. Such semi-supervised settings are denoted as \textit{Dataset-U}. In this setting, robust hinge loss is used for both approaches in the case of unlabeled nodes, while labeled nodes are treated as indicated in \autoref{tab:robust_training} with robust hinge or BCE loss.
We use the same GCN architecture and robust training procedure as in \textit{\Dual{}} \cite{singh2019abstract} with a global perturbation budget of 12. We divide the data into $10\%$ training and $90\%$ testing, averaging the result of five random splits. 

\noindent\textbf{Analysis:}
Robust training results are shown in \autoref{tab:robust_training}. Columns show robustness of lower bound and test accuracy for datasets in supervised and semi-supervised settings. The first row shows results for cross-entropy loss and the following rows show training performance with robust losses. 

Compared to the standard CE loss, the robust training improves the robustness significantly in \textit{all} cases while keeping the accuracy almost unaffected in both \textit{\hinge} and \textit{\bce{}}. Moreover, semi-supervised training improves the robustness significantly for both Citeseer and Cora datasets compared to the supervised setting. Interestingly, the robustness of PubMed-U is only slightly enhanced compared to PubMed. The small improvement may be because PubMed contains many nodes and even $10\%$ of nodes is enough for the GCN to be trained near its limit. Thus, the semi-supervised setting provides more benefits on relatively small datasets like Citeseer and Cora.

We carry out statistical analysis to compare the performance of \textit{\Ours{}} and \textit{\Dual{}} under the same robust training setting across different datasets. We first run the training for 5 random seeds and record all metric values of the trained model. Then, we calculate the Cohen-d effect size \cite{cohen2013statistical} and highlight (in bold) the ones with a larger than 0.8 Cohen-d value (corresponding to a large effect size). The performance of \textit{\Ours{}} and \textit{\Dual} is similar in robust training. For robustness, \textit{\Dual} is better in 3/12 cases, and the two approaches are comparable in 9/12 cases. \textit{\Dual} also has a higher test accuracy in 2/12 cases, and the two approaches are  similar 
in 10/12 cases.  

\chadded{
We hypothesize that the slightly weaker robust training performance of \textit{\Ours{}} may be caused by a vanishing gradient problem~\cite{hochreiter1998vanishing} associated with the lower bound used in bounding $\relu$. As shown in \autoref{thm:relu_bounding}, one of the lower bounds for $\relu$ is $0$. The frequent occurrence of this lower bound may lead to flattened gradients, thereby impacting training performance. To test this hypothesis, we replaced the $0$ lower bound with $\lambda x$ during training, where $\lambda$ is a small positive constant, in order to preserve nonzero gradients. Note that this adjustment remains sound, since $\lambda x \leq 0$ when $x \leq 0$. However, under the same training configuration and using the \bce{} loss, we observed no significant performance change across datasets for $\lambda \in \{0.1, 0.5\}$. Thus, further research will be needed to improve the robust training performance of \textit{\Ours{}}.
}


\begin{tcolorbox}
\textbf{Answer to RQ4.}  
The \Ours{} approach enhances the robustness of the GCN model when incorporated during robust training. Furthermore, this enhancement is comparable with the contemporary state-of-the-art robust training techniques.
\end{tcolorbox}

\subsection{Threats to validity}
\chadded{\textbf{Internal validity.}
Internal validity concerns the experimental setup. Specifically, we observed potential imprecision in the \Dual{} approach caused by sparse matrix multiplications. Although the fixed version improves the tightness of the upper bound, we report results for both versions to ensure completeness. In addition, the runtime of each approach may fluctuate across runs. To mitigate this variability, we executed all approaches in the same environment and report the median over ten runs. For robust training, GCN training involves several random factors, such as dataset shuffling. To address this threat, we report average metric values over five random seeds.
}

\noindent
\textbf{External validity.}
\chadded{
The main external validity threat of this work lies in the datasets used for evaluation. Although widely adopted in prior studies \cite{liu2020abstract,zugner2019certifiable,lee2019tight}, these datasets primarily contain homogeneous graphs of small to medium size. This choice reflects the design limitations of GCNs, which operate on homogeneous graphs. In contrast, many real-world large-scale graphs, such as knowledge graphs (e.g., Chameleon, Squirrel, and Texas~\cite{pei2020geom,rozemberczki2021multi}), are heterogeneous, containing multiple node and edge types. Extending the evaluation to such settings would require adapting the approach to alternative GNN architectures. We therefore leave this direction to future work.
}

\chadded{
\noindent\textbf{Construct validity.}
This threat mostly concerns metric selection in this study. For all experiments, we adopt widely used metrics, such as the certified lower bound for robustness tightness and accuracy for node classification. In addition, we introduce a new metric, the uncertainty region, which provides a holistic assessment of a certifier’s performance across a range of perturbation limits.
}

\section{Related work}
\label{sec:rel-work}

\subsection{Graph adversarial examples}
The robustness of neural networks against adversarial examples has been studied extensively \cite{szegedy2013intriguing}. Graph adversarial examples also receive significant interest, and multiple surveys exist on this topic \cite{sun2022adversarial,chen2020survey,GNNBook-ch8-gunnemann}. Two main types of approaches exist for this task: node-level and edge-level. 

\textbf{Node-level} adversarial examples focus on either modifying node attributes or adding nodes to the target graph. The task type determines the approaches used to select nodes or node features to change. Many adversarial attacks by changing node features use GNN gradients to increase task-specific loss \cite{zugner2018adversarial,takahashi2019indirect,ma2020towards}. Methods for fake node injection typically use reinforcement learning \cite{sun2020non} and optimization \cite{tao2021single,zou2021tdgia}. Generative models such as generative adversarial networks (GAN) have also been applied to both node feature perturbations \cite{bose2019generalizable} and node injection \cite{wang2018attack}.

\textbf{Edge-level} changes such as adding, removing, or rewiring edges in the graph can also create adversarial examples for GNNs. Reinforcement learning-based approaches are used to learn to select the most influential edge to modify \cite{ma2019attacking,dai2018adversarial}. Some meta-learning \cite{zugner2019adversarial} and optimization-based methods \cite{wang2019attacking,xu2019topology,zou2021tdgia} are also used for adversarial construction with edge addition or removal. 

At the same time, different defense methods against adversarial examples have been proposed. Many of the methods improve the training of GNNs. For example, the \Dual{} \cite{zugner2019certifiable} baseline used in this paper can improve the robustness of GNNs by showing the lower bound of robustness during training. Other methods improve the robustness by focusing on specific graph structures during training \cite{miller2024complex}, or even modifying the GNN architecture \cite{chen2021understanding}. Tao et. al. \cite{tao2023graph} proposed a method to fix and immunize certain parts of the graph from any modification to improve the robustness of GNNs.

Counterexamples derived from our approach can be seen as node-level adversarial examples for GCNs. But our approach also provides certified lower bounds of the GCN being robust under \textit{any} node-level graph adversarial examples. Furthermore, our method can be used to improve the robustness during training, similar to the \Dual{} method.

\subsection{Neural networks certification}
Many approaches have been proposed to certify properties of neural networks. These approaches can be categorized by the types of abstraction they use during the verification. Two surveys provide comprehensive overviews \cite{liu2021algorithms,albarghouthi2021introduction}. 

\textbf{Optimization-based} methods abstract the neural network into a set of optimization clauses where the property is formulated as the objective function. Then an external optimization solver can be utilized to check the property can be satisfied. NSVerify \cite{meyer2022reachability} and MIPVerify \cite{tjeng2017evaluating} translate the verification problem into a mixed integer programming problem. ILP \cite{bastani2016measuring} uses iterative linear programming to verify the robustness of neural networks. 
More recently, mm-bab \cite{ferrari2022complete} combines branch-and-bound and dual optimization to certify neural networks efficiently on GPUs. 


\textbf{Abstract interpretation-based} methods typically define a domain to capture the input of the neural network as well as the transformation operations of elements in the abstract domain. Deeppoly \cite{singh2019abstract} uses the Deeppoly abstract domain to verify convolutional neural networks with ReLU or Sigmoid activation functions. ReluVal \cite{wang2018formal} and Neurify \cite{wang2018efficient} use symbolic intervals to capture the set of inputs for neural networks. 

\textbf{Hybrid methods} combine different approaches to design a hybrid verifier. Alpha-beta crowning \cite{zhang2018efficient,xu2021fast,wang2021beta} combines bound propagation and branch and bound techniques for fast and precise certification on neural networks. 



\subsection{Graph neural network certification}
\begin{table}[]
    \makebox[\linewidth][c]{
    \scalebox{0.9}{
    \begin{tabular}{|c|c|c|c|c|c|}
        \hline
         Approach & Certification Method & Guarantee & Certification Target & Tightness & Speed  \\
        \hline
        \cite{lee2019tight,dvijotham2020framework,bojchevski2020efficient,jia2020certified,scholten2024hierarchical} & \makecell{Randomized smoothing} & Probabilistic & \makecell{Smoothed \\black-box GNN} & \makecell{Tight} & Moderate to fast \\
        \hline
        \Dual{} \cite{zugner2019certifiable} & Dual optimization & Absolute & GCN & Tight & Moderate \\ 
        \hline
        \Interval{} \cite{liu2020abstract} & Abstract interpretation & Absolute & GCN & Loose & Fast \\
        \hline
        \Ours{} (Ours) & Abstract interpretation & Absolute & GCN & Tightest & Fast \\
        \hline
    \end{tabular}
    }
    }
    \caption{Comparison of GNN Certification Approaches}
    \label{tab:approach-compare}
\end{table}

Though many certification methods exist for other neural networks, they do not directly apply to GNNs due to the highly relational nature of graphs. Thus, methods tailored for GNNs have emerged. Recent surveys \cite{GNNBook-ch8-gunnemann,zhai2023state} provide a comprehensive review of these methods, which can be categorized by the type of guarantee provided by the approaches. \autoref{tab:approach-compare} compares different GNN certification approaches.

\textbf{Probabilistic certificate.} These certification provides \textit{probabilistic guarantees} on the robustness in the form of a confident interval. Most of these methods adapt randomized smoothing  \cite{cohen2019certified}, a method assuming the GNN is a black box, from continuous data to graph data \cite{lee2019tight,dvijotham2020framework,bojchevski2020efficient,jia2020certified}. Moreover, Scholten et. al. \cite{scholten2024hierarchical} propose a hierarchical randomized smoothing method that works for both continuous and graph data. However, randomized smoothing methods modify the input of the original GNN to create a smoothed version of it. Thus, it provides a probabilistic guarantee on the robustness of the \textit{smoothed} GNN rather than the original GNN.

\textbf{Absolute certification.} In practice, it is usually preferred to derive absolute certification on the robustness of the original GCN. These methods providing absolute guarantee are typically designed for a specific GNN architecture. The most studied GNN architecture is GCNs. For example, the Dual \cite{zugner2019certifiable} and Interval \cite{liu2020abstract} methods compared in this paper are both model-specific methods with GCNs and node feature perturbation. Other certification methods mainly focus on edge perturbation\cite{zugner2020certifiable,jin2020certified} other types of tasks \cite{osselin2023structure}, or other GNN architectures \cite{bojchevski2019certifiable,hojny2024verifying}. 

This paper proposes a novel absolute certification approach for GCNs with node perturbation by applying polyhedra abstract interpretation techniques to graph data. Compared to abstract interpretation approaches of neural networks for non-relational continuous input, we also tackle several challenges unique to graph data. It can derive the lower and upper bounds of the robustness at the same time. Moreover, it simultaneously improves the certification tightness and the runtime performance compared with existing certification approaches for GCNs.  

Recently, the notion of \emph{expected robustness} and \emph{collective robustness} \cite{schuchardt2023collective}  is proposed \cite{abbahaddou2024bounding}. The expected robustness can be seen as a more general definition of the "worst-case" robustness investigated in this paper while the abstract interpretation method can be used to derive more precise collective robustness of GCNs.

\section{Conclusions}
\label{sec:conclusions}
In this study, we tackle the challenge of GCN certification for node classification in the presence of node feature perturbation. According to the perturbation model, each node can be altered within a specific local limit while the entire graph is subject to changes within a global limit. We introduce generic concepts for sound and complete certifiers for node classification. We propose a polyhedra \absint{} method for certifying GCN robustness within a defined perturbation region, ensuring \textit{no} changes to the original graph can influence the GCN's prediction. Our experimental evaluations, conducted on three widely used node classification datasets and compared with two other certification methods, indicate that our method provides tighter certification bounds and significantly improves the certification runtime. This improvement in the tightness is further reflected in the collective certification. Moreover, our approach demonstrates promising results in improving GCN robustness through robust training. 

We believe our approach introduces a novel direction in GNN certification. While our polyhedra abstract interpretation presents strong advantages for GCN certification in node classification, it is important to note that real-world scenarios might encounter \chadded{larger-scale heterogeneous graphs, other types of graph perturbations, diverse GNN architectures}, and different graph tasks. One direction for future work involves investigating using \absint{} in these contexts. Notably, despite the better tightness of our approach, it does not outperform baselines in robust training. Thus, subsequent research could target to optimize robustness training methods to further improve the robustness of GCNs. Finally, an added merit of our method is its ability to express the GCN results as linear expressions. These expressions can potentially be used to enhance the interpretability of GCNs.


\bibliography{sample-base}

@String{Computing = "Computing" }

@String{Computer = "{IEEE} Computer" }

@String{Springer = "Springer-Verlag" }

@article{liu2021algorithms,
  title={Algorithms for verifying deep neural networks},
  author={Liu, Changliu and Arnon, Tomer and Lazarus, Christopher and Strong, Christopher and Barrett, Clark and Kochenderfer, Mykel J and others},
  journal={Foundations and Trends in Optimization},
  volume={4},
  number={3-4},
  pages={244--404},
  year={2021},
  publisher={Now Publishers, Inc.}
}

@article{albarghouthi2021introduction,
  title={Introduction to neural network verification},
  author={Albarghouthi, Aws},
  journal={Foundations and Trends in Programming Languages},
  volume={7},
  number={1--2},
  pages={1--157},
  year={2021},
  publisher={Now Publishers, Inc.}
}

@article{meyer2022reachability,
  title={Reachability analysis of neural networks using mixed monotonicity},
  author={Meyer, Pierre-Jean},
  journal={IEEE Control Systems Letters},
  volume={6},
  pages={3068--3073},
  year={2022},
  publisher={IEEE}
}

@article{tjeng2017evaluating,
  title={Evaluating robustness of neural networks with mixed integer programming},
  author={Tjeng, Vincent and Xiao, Kai and Tedrake, Russ},
  journal={International Conference on Learning Representations},
  year={2019}
}

@inproceedings{bastani2016measuring,
author = {Bastani, Osbert and Ioannou, Yani and Lampropoulos, Leonidas and Vytiniotis, Dimitrios and Nori, Aditya V. and Criminisi, Antonio},
title = {Measuring Neural Net Robustness with Constraints},
year = {2016},
booktitle = {Proceedings of the 30th International Conference on Neural Information Processing Systems},
pages = {2621–2629},
numpages = {9},
location = {Barcelona, Spain},
series = {NIPS'16}
}

@inproceedings{zugner2019certifiable,
  title={Certifiable robustness and robust training for graph convolutional networks},
  author={Z{\"u}gner, Daniel and G{\"u}nnemann, Stephan},
  booktitle={Proceedings of the 25th ACM SIGKDD International Conference on Knowledge Discovery \& Data Mining},
  pages={246--256},
  year={2019}
}

@inproceedings{zugner2020certifiable,
  title={Certifiable robustness of graph convolutional networks under structure perturbations},
  author={Z{\"u}gner, Daniel and G{\"u}nnemann, Stephan},
  booktitle={Proceedings of the 26th ACM SIGKDD international conference on knowledge discovery \& data mining},
  pages={1656--1665},
  year={2020}
}

@article{singh2019abstract,
  title={An abstract domain for certifying neural networks},
  author={Singh, Gagandeep and Gehr, Timon and P{\"u}schel, Markus and Vechev, Martin},
  journal={Proceedings of the ACM on Programming Languages},
  volume={3},
  number={POPL},
  pages={1--30},
  year={2019},
  publisher={ACM New York, NY, USA}
}

@inproceedings{wang2018formal,
  title={Formal security analysis of neural networks using symbolic intervals},
  author={Wang, Shiqi and Pei, Kexin and Whitehouse, Justin and Yang, Junfeng and Jana, Suman},
  booktitle={27th USENIX Security Symposium (USENIX Security 18)},
  pages={1599--1614},
  year={2018}
}

@article{wang2018efficient,
  title={Efficient formal safety analysis of neural networks},
  author={Wang, Shiqi and Pei, Kexin and Whitehouse, Justin and Yang, Junfeng and Jana, Suman},
  journal={Advances in Neural Information Processing Systems},
  volume={31},
  pages = {6369–6379},
  year={2018}
}

@article{liu2020abstract,
  title={Abstract interpretation based robustness certification for graph convolutional networks},
  author={Liu, Yang and Peng, Jiaying and Chen, Liang and Zheng, Zibin},
  journal={European Conference on Artificial Intelligence 2020},
  pages={1309-1315},
  year={2020},
}

@article{szegedy2013intriguing,
  title={Intriguing properties of neural networks},
  author={Szegedy, Christian and Zaremba, Wojciech and Sutskever, Ilya and Bruna, Joan and Erhan, Dumitru and Goodfellow, Ian and Fergus, Rob},
  journal={arXiv preprint arXiv:1312.6199},
  year={2013}
}

@article{sun2022adversarial,
  title={Adversarial attack and defense on graph data: A survey},
  author={Sun, Lichao and Dou, Yingtong and Yang, Carl and Zhang, Kai and Wang, Ji and Philip, S Yu and He, Lifang and Li, Bo},
  journal={IEEE Transactions on Knowledge and Data Engineering},
  volume={35},
  number={8},
  pages={7693--7711},
  year={2022},
  publisher={IEEE}
}

@article{chen2020survey,
  title={A survey of adversarial learning on graphs},
  author={Chen, Liang and Li, Jintang and Peng, Jiaying and Xie, Tao and Cao, Zengxu and Xu, Kun and He, Xiangnan and Zheng, Zibin},
  journal={arXiv preprint arXiv:2003.05730},
  year={2020}
}

@incollection{GNNBook-ch8-gunnemann,
author = "G{\"u}nnemann, Stephan",
editor = "Wu, Lingfei and Cui, Peng and Pei, Jian and Zhao, Liang",
title = "Graph Neural Networks: Adversarial Robustness",
booktitle = "Graph Neural Networks: Foundations, Frontiers, and Applications",
year = "2022",
publisher = "Springer Singapore",
address = "Singapore",
pages = "149--176",
}

@inproceedings{zugner2018adversarial,
  title={Adversarial attacks on neural networks for graph data},
  author={Z{\"u}gner, Daniel and Akbarnejad, Amir and G{\"u}nnemann, Stephan},
  booktitle={Proceedings of the 24th ACM SIGKDD international conference on knowledge discovery \& data mining},
  pages={2847--2856},
  year={2018}
}

@inproceedings{takahashi2019indirect,
  title={Indirect adversarial attacks via poisoning neighbors for graph convolutional networks},
  author={Takahashi, Tsubasa},
  booktitle={2019 IEEE International Conference on Big Data (Big Data)},
  pages={1395--1400},
  year={2019},
  organization={IEEE}
}

@article{ma2020towards,
  title={Towards more practical adversarial attacks on graph neural networks},
  author={Ma, Jiaqi and Ding, Shuangrui and Mei, Qiaozhu},
  journal={Advances in neural information processing systems},
  volume={33},
  pages={4756--4766},
  year={2020}
}

@article{bose2019generalizable,
  title={Generalizable adversarial attacks with latent variable perturbation modelling},
  author={Bose, Avishek Joey and Cianflone, Andre and Hamilton, William L},
  journal={arXiv preprint arXiv:1905.10864},
  year={2019}
}

@inproceedings{sun2020non,
author = {Sun, Yiwei and Wang, Suhang and Tang, Xianfeng and Hsieh, Tsung-Yu and Honavar, Vasant},
title = {Adversarial Attacks on Graph Neural Networks via Node Injections: A Hierarchical Reinforcement Learning Approach},
year = {2020},
booktitle = {Proceedings of The Web Conference 2020},
pages = {673–683},
numpages = {11},
series = {WWW'20}
}

@inproceedings{tao2021single,
  title={Single node injection attack against graph neural networks},
  author={Tao, Shuchang and Cao, Qi and Shen, Huawei and Huang, Junjie and Wu, Yunfan and Cheng, Xueqi},
  booktitle={Proceedings of the 30th ACM International Conference on Information \& Knowledge Management},
  pages={1794--1803},
  year={2021}
}

@inproceedings{zou2021tdgia,
  title={TDGIA: Effective injection attacks on graph neural networks},
  author={Zou, Xu and Zheng, Qinkai and Dong, Yuxiao and Guan, Xinyu and Kharlamov, Evgeny and Lu, Jialiang and Tang, Jie},
  booktitle={Proceedings of the 27th ACM SIGKDD Conference on Knowledge Discovery \& Data Mining},
  pages={2461--2471},
  year={2021}
}

@article{wang2018attack,
  title={Attack graph convolutional networks by adding fake nodes},
  author={Wang, Xiaoyun and Cheng, Minhao and Eaton, Joe and Hsieh, Cho-Jui and Wu, Felix},
  journal={arXiv preprint arXiv:1810.10751},
  year={2018}
}

@article{ma2019attacking,
  title={Attacking graph convolutional networks via rewiring},
  author={Ma, Yao and Wang, Suhang and Derr, Tyler and Wu, Lingfei and Tang, Jiliang},
  journal={arXiv preprint arXiv:1906.03750},
  year={2019}
}

@inproceedings{dai2018adversarial,
  title={Adversarial attack on graph structured data},
  author={Dai, Hanjun and Li, Hui and Tian, Tian and Huang, Xin and Wang, Lin and Zhu, Jun and Song, Le},
  booktitle={International conference on machine learning},
  pages={1115--1124},
  year={2018},
  organization={PMLR}
}

@article{zugner2019adversarial,
  author    = {Daniel Z{\"{u}}gner and
               Stephan G{\"{u}}nnemann},
  title     = {Adversarial Attacks on Graph Neural Networks via Meta Learning},
  journal = {7th International Conference on Learning Representations, {ICLR} 2019,
               New Orleans, LA, USA, May 6-9, 2019},
  publisher = {OpenReview.net},
  year      = {2019},
}

@inproceedings{wang2019attacking,
  title={Attacking graph-based classification via manipulating the graph structure},
  author={Wang, Binghui and Gong, Neil Zhenqiang},
  booktitle={Proceedings of the 2019 ACM SIGSAC Conference on Computer and Communications Security},
  pages={2023--2040},
  year={2019}
}

@article{xu2019topology,
  title={Topology attack and defense for graph neural networks: An optimization perspective},
  author={Xu, Kaidi and Chen, Hongge and Liu, Sijia and Chen, Pin-Yu and Weng, Tsui-Wei and Hong, Mingyi and Lin, Xue},
  journal={International Joint Conference on Artificial Intelligence},
  year={2019}
}

@inproceedings{wang2021certified,
  title={Certified robustness of graph neural networks against adversarial structural perturbation},
  author={Wang, Binghui and Jia, Jinyuan and Cao, Xiaoyu and Gong, Neil Zhenqiang},
  booktitle={Proceedings of the 27th ACM SIGKDD Conference on Knowledge Discovery \& Data Mining},
  pages={1645--1653},
  year={2021}
}

@inproceedings{cohen2019certified,
  title={Certified adversarial robustness via randomized smoothing},
  author={Cohen, Jeremy and Rosenfeld, Elan and Kolter, Zico},
  booktitle={international conference on machine learning},
  pages={1310--1320},
  year={2019},
  organization={PMLR}
}

@article{kipf2016semi,
  title={Semi-supervised classification with graph convolutional networks},
  author={Kipf, Thomas N and Welling, Max},
  journal={arXiv preprint arXiv:1609.02907},
  year={2016}
}

@inproceedings{cousot1977abstract,
  title={Abstract interpretation: a unified lattice model for static analysis of programs by construction or approximation of fixpoints},
  author={Cousot, Patrick and Cousot, Radhia},
  booktitle={Proceedings of the 4th ACM SIGACT-SIGPLAN Symposium on Principles of Programming Languages},
  pages={238--252},
  year={1977}
}

@article{urban2020perfectly,
  title={Perfectly parallel fairness certification of neural networks},
  author={Urban, Caterina and Christakis, Maria and W{\"u}stholz, Valentin and Zhang, Fuyuan},
  journal={Proceedings of the ACM on Programming Languages},
  volume={4},
  number={OOPSLA},
  pages={1--30},
  year={2020},
  publisher={ACM New York, NY, USA}
}

@article{mccallum2000automating,
  title={Automating the construction of internet portals with machine learning},
  author={McCallum, Andrew Kachites and Nigam, Kamal and Rennie, Jason and Seymore, Kristie},
  journal={Information Retrieval},
  volume={3},
  pages={127--163},
  year={2000},
  publisher={Springer}
}

@article{sen2008collective,
  title={Collective classification in network data},
  author={Sen, Prithviraj and Namata, Galileo and Bilgic, Mustafa and Getoor, Lise and Galligher, Brian and Eliassi-Rad, Tina},
  journal={AI Magazine},
  volume={29},
  number={3},
  pages={93--93},
  year={2008}
}

@inproceedings{giles1998citeseer,
  title={CiteSeer: An automatic citation indexing system},
  author={Giles, C Lee and Bollacker, Kurt D and Lawrence, Steve},
  booktitle={Proceedings of the third ACM conference on Digital libraries},
  pages={89--98},
  year={1998}
}

@article{cohen2013statistical,
  title={Statistical power analysis for the behavioral sciences},
  author={Cohen, Jacob},
  year={2013},
  journal={Routledge}
}

@inproceedings{cousot1978automatic,
  title={Automatic discovery of linear restraints among variables of a program},
  author={Cousot, Patrick and Halbwachs, Nicolas},
  booktitle={Proceedings of the 5th ACM SIGACT-SIGPLAN Symposium on Principles of Programming Languages},
  pages={84--96},
  year={1978}
}

@inproceedings{dettmers2018convolutional,
  title={Convolutional 2d knowledge graph embeddings},
  author={Dettmers, Tim and Minervini, Pasquale and Stenetorp, Pontus and Riedel, Sebastian},
  booktitle={Proceedings of the AAAI conference on artificial intelligence},
  volume={32},
  year={2018},
  pages={1811-1818}
}

@article{hochreiter1998vanishing,
  title={The vanishing gradient problem during learning recurrent neural nets and problem solutions},
  author={Hochreiter, Sepp},
  journal={International Journal of Uncertainty, Fuzziness and Knowledge-Based Systems},
  volume={6},
  number={02},
  pages={107--116},
  year={1998},
  publisher={World Scientific}
}

@INPROCEEDINGS{sun-icus2020,
  author={Sun, Mingjie and Jiang, Yinan and Wang, Yashen and Xie, Haiyong and Wang, Zhaobin},
  booktitle={2020 3rd International Conference on Unmanned Systems (ICUS)}, 
  title={Identification of Critical Nodes in Dynamic Systems Based on Graph Convolutional Networks}, 
  year={2020},
  pages={558-563},
  OPTdoi={10.1109/ICUS50048.2020.9274812}
}

@article{Yu-KBS2020,
title = {Identifying critical nodes in complex networks via graph convolutional networks},
journal = {Knowledge-Based Systems},
volume = {198},
pages = {105893},
year = {2020},
OPTissn = {0950-7051},
OPTdoi = {https://doi.org/10.1016/j.knosys.2020.105893},
OPTurl = {https://www.sciencedirect.com/science/article/pii/S0950705120302409},
author = {En-Yu Yu and Yue-Ping Wang and Yan Fu and Duan-Bing Chen and Mei Xie},
}

@article{rgcn-intro2021,
  title={An Introduction to Robust Graph Convolutional Networks},
  author={Najafi, Mehrnaz and Yu, Philip S},
  journal={arXiv preprint arXiv:2103.14807},
  year={2021}
}

@inproceedings{salzer2021reachability,
  title={Reachability is NP-complete even for the simplest neural networks},
  author={S{\"a}lzer, Marco and Lange, Martin},
  booktitle={Reachability Problems: 15th International Conference, RP 2021, Liverpool, UK, October 25--27, 2021, Proceedings 15},
  pages={149--164},
  year={2021},
  organization={Springer}
}

@inproceedings{xu2021fast,
    title={{Fast and Complete}: Enabling Complete Neural Network Verification with Rapid and Massively Parallel Incomplete Verifiers},
    author={Kaidi Xu and Huan Zhang and Shiqi Wang and Yihan Wang and Suman Jana and Xue Lin and Cho-Jui Hsieh},
    booktitle={International Conference on Learning Representations},
    year={2021},
    OPTurl={https://openreview.net/forum?id=nVZtXBI6LNn}
}

@article{wang2021beta,
  title={{Beta-CROWN}: Efficient bound propagation with per-neuron split constraints for complete and incomplete neural network verification},
  author={Wang, Shiqi and Zhang, Huan and Xu, Kaidi and Lin, Xue and Jana, Suman and Hsieh, Cho-Jui and Kolter, J Zico},
  journal={Advances in Neural Information Processing Systems},
  volume={34},
  year={2021}
}

@article{zhang2018efficient,
  title={Efficient Neural Network Robustness Certification with General Activation Functions},
  author={Zhang, Huan and Weng, Tsui-Wei and Chen, Pin-Yu and Hsieh, Cho-Jui and Daniel, Luca},
  journal={Advances in Neural Information Processing Systems},
  volume={31},
  pages={4939--4948},
  year={2018},
  OPTurl={https://arxiv.org/pdf/1811.00866.pdf}
}

@article{ferrari2022complete,
  title={Complete verification via multi-neuron relaxation guided branch-and-bound},
  author={Ferrari, Claudio and Muller, Mark Niklas and Jovanovic, Nikola and Vechev, Martin},
  journal={International Conference on Learning Representations},
  year={2022}
}

@article{jin2020certified,
  title={Certified robustness of graph convolution networks for graph classification under topological attacks},
  author={Jin, Hongwei and Shi, Zhan and Peruri, Venkata Jaya Shankar Ashish and Zhang, Xinhua},
  journal={Advances in neural information processing systems},
  volume={33},
  pages={8463--8474},
  year={2020}
}

@article{bojchevski2019certifiable,
  title={Certifiable robustness to graph perturbations},
  author={Bojchevski, Aleksandar and G{\"u}nnemann, Stephan},
  journal={Advances in Neural Information Processing Systems},
  volume={32},
  year={2019}
}

@article{lee2019tight,
  title={Tight certificates of adversarial robustness for randomly smoothed classifiers},
  author={Lee, Guang-He and Yuan, Yang and Chang, Shiyu and Jaakkola, Tommi},
  journal={Advances in Neural Information Processing Systems},
  volume={32},
  year={2019}
}

@article{dvijotham2020framework,
  title={A framework for robustness certification of smoothed classifiers using f-divergences},
  author={Dvijotham, Krishnamurthy Dj and Hayes, Jamie and Balle, Borja and Kolter, Zico and Qin, Chongli and Gyorgy, Andras and Xiao, Kai and Gowal, Sven and Kohli, Pushmeet},
  year={2020}
}

@inproceedings{bojchevski2020efficient,
  title={Efficient robustness certificates for discrete data: Sparsity-aware randomized smoothing for graphs, images and more},
  author={Bojchevski, Aleksandar and Gasteiger, Johannes and G{\"u}nnemann, Stephan},
  booktitle={International Conference on Machine Learning},
  pages={1003--1013},
  year={2020},
  organization={PMLR}
}

@inproceedings{jia2020certified,
  title={Certified robustness of community detection against adversarial structural perturbation via randomized smoothing},
  author={Jia, Jinyuan and Wang, Binghui and Cao, Xiaoyu and Gong, Neil Zhenqiang},
  booktitle={Proceedings of The Web Conference 2020},
  pages={2718--2724},
  year={2020}
}

@inproceedings{
schuchardt2023collective,
title={Collective Robustness Certificates: Exploiting Interdependence in Graph Neural Networks},
author={Jan Schuchardt and Aleksandar Bojchevski and Johannes Gasteiger and Stephan G{\"u}nnemann},
booktitle={International Conference on Learning Representations},
year={2021},
OPTurl={https://openreview.net/forum?id=ULQdiUTHe3y}
}

@article{singh2018fast,
  title={Fast and effective robustness certification},
  author={Singh, Gagandeep and Gehr, Timon and Mirman, Matthew and P{\"u}schel, Markus and Vechev, Martin},
  journal={Advances in neural information processing systems},
  volume={31},
  year={2018}
}

@article{zhai2023state,
  title={State of the art on adversarial attacks and defenses in graphs},
  author={Zhai, Zhengli and Li, Penghui and Feng, Shu},
  journal={Neural Computing and Applications},
  volume={35},
  number={26},
  pages={18851--18872},
  year={2023},
  publisher={Springer}
}

@article{scholten2024hierarchical,
  title={Hierarchical randomized smoothing},
  author={Scholten, Yan and Schuchardt, Jan and Bojchevski, Aleksandar and G{\"u}nnemann, Stephan},
  journal={Advances in Neural Information Processing Systems},
  volume={36},
  year={2024}
}

@article{hojny2024verifying,
  title={Verifying message-passing neural networks via topology-based bounds tightening},
  author={Hojny, Christopher and Zhang, Shiqiang and Campos, Juan S and Misener, Ruth},
  journal={arXiv preprint arXiv:2402.13937},
  year={2024}
}

@article{miller2024complex,
  title={Complex network effects on the robustness of graph convolutional networks},
  author={Miller, Benjamin A and Chan, Kevin and Eliassi-Rad, Tina},
  journal={Applied Network Science},
  volume={9},
  number={1},
  pages={5},
  year={2024},
  publisher={Springer}
}

@article{chen2021understanding,
  title={Understanding structural vulnerability in graph convolutional networks},
  author={Chen, Liang and Li, Jintang and Peng, Qibiao and Liu, Yang and Zheng, Zibin and Yang, Carl},
  journal={arXiv preprint arXiv:2108.06280},
  year={2021}
}

@article{tao2023graph,
  title={Graph adversarial immunization for certifiable robustness},
  author={Tao, Shuchang and Cao, Qi and Shen, Huawei and Wu, Yunfan and Hou, Liang and Cheng, Xueqi},
  journal={IEEE Transactions on Knowledge and Data Engineering},
  year={2023},
  publisher={IEEE}
}

@inproceedings{osselin2023structure,
  title={Structure-aware robustness certificates for graph classification},
  author={Osselin, Pierre and Kenlay, Henry and Dong, Xiaowen},
  booktitle={Uncertainty in Artificial Intelligence},
  pages={1596--1605},
  year={2023},
  organization={PMLR}
}

@inproceedings{abbahaddou2024bounding,
  title={Bounding the Expected Robustness of Graph Neural Networks Subject to Node Feature Attacks},
  author={Abbahaddou, Yassine and Ennadir, Sofiane and Lutzeyer, Johannes F and Vazirgiannis, Michalis and Bostr{\"o}m, Henrik},
  booktitle={International Conference on Learning Representations (ICLR)},
  year={2024}
}

@article{rozemberczki2021multi,
  author       = {Benedek Rozemberczki and
                  Carl Allen and
                  Rik Sarkar},
  title        = {Multi-Scale attributed node embedding},
  journal      = {J. Complex Networks},
  volume       = {9},
  number       = {2},
  year         = {2021},
}

@inproceedings{pei2020geom,
title={Geom-GCN: Geometric Graph Convolutional Networks},
author={Hongbin Pei and Bingzhe Wei and Kevin Chen-Chuan Chang and Yu Lei and Bo Yang},
booktitle={International Conference on Learning Representations},
year={2020},
}

\clearpage
\appendix
\appendix
\section{Appendix}
\label{sec:appendix}
\subsection{Interval abstract interpretation}
\label{append:interval}
We can define the interval \absint{} \cite{liu2020abstract} approach using our general \absint{} framework. The interval abstract domain \cite{cousot1977abstract} captures the set of inputs by interval bounds. For the concrete domain of feature vectors $\concreteElement \in \concreteDomain$, the abstract domain can be defined as a tuple of matrices: $\mathbb{A}_I=(\realDomain^{(\numNodes, \numFeatures_\layer{})}, \realDomain^{(\numNodes, \numFeatures_\layer)})$, such that for any $(\lowBound, \upBound) \in \mathbb{A}_I$, it satisfies $\lowBound \le \upBound$ for each entry of the matrices. 

\paragraph{Input abstraction}
Given the initial perturbation space $\perturba{\nodeF}$, the input abstraction can be defined as $\absele_0 = (min(\mathcal{\perturba{\nodeF}}), max(\mathcal{\perturba{\nodeF}}))$, where $min$ and $max$ finds the element-wise minimum and maximum, respectively. 
However, directly abstracting the input domain $\wp(\nodeFeature{B}{\numNodes}{\numFeatures_0})$ is not helpful because this will result the lower bound to be always zero and upper bound to be always one for the given $\perturba{X}$ since all the feature can be potentially flipped. Such abstraction will be too imprecise to analyze. To make the analysis more precise, one can instead abstract from the concrete domain after the first GNN layer $\wp(\nodeFeature{R}{\numNodes}{\numFeatures_1})$. 

Let $\concreteElement_0 \in \wp(\nodeFeature{B}{\numNodes}{\numFeatures_0}), \concreteElement_1 \in \wp(\nodeFeature{R}{\numNodes}{\numFeatures_1})$ be the values at the input layer and after the first GCN layer, before the $\relu$ operation by $\concreteElement_1 = \overline{\linear{\weight, \bias}}(\overline{\gc{}}(\concreteElement_0))$. The input abstraction can then be defined using $(max(\concreteElement_1)$ and $min(\concreteElement_1))$ based on $\concreteElement_0$ and the first GCN layer \cite{liu2020abstract,zugner2019certifiable}. Specifically, let $\nodeF$ be original node feature matrix and $\latentFeatures_1$ be the feature value calculated by $\latentFeatures_1 = \gc(\linear{W,b}(X))$ $\localp$ be the local perturbation budget and $\globalp$ be the global perturbation budget; we can compute the maximum by using the top perturbations for the first layer:

\begin{equation*}
\label{equ:activation_bounding}
\begin{aligned}
&max(\concreteElement_1)_{\node,j} = (\sum \op{MaxK}_\globalp(flatten(diag(\Tilde{A_{\node:}})\hat{S}_j))) + H^1_{\node, j}, \\
&\text{ where }  \hat{S}_j = max(\op{MaxK}_\localp(P\:diag(W_{:j})), 0)
\end{aligned}
\end{equation*}
$P = (1 - X) - X$ is the matrix of all possible perturbations, $\op{MaxK}_k$ gets the top $k$ elements for each row of the matrix while we only consider positive values using $max(\cdot, 0)$, $diag$ creates a diagonal matrix with the vector and $flatten$ flattens the matrix into a row vector.

Finding $min(\mathcal{S}_1)$ is similar by replacing the $\op{MaxK}$ and $max(\cdot, 0)$ functions with $\op{MinK}$ and $min(\cdot, 0)$ for the minimum values. Finally, the abstraction function can be defined as $(\lowBound^{TopK}_1, \upBound^{TopK}_1) =(min(\mathcal{S}_1), max(\mathcal{S}_1))$.

However, this method may have an impact on the runtime since the $\op{MaxK}$ and $\op{MinK}$ require (partially) sorting the input. Another variation of the method is to find the maximum/minimum value. 
\begin{equation*}
\label{equ:activation_bounding_max}
\begin{aligned}
&max(\concreteElement_1)_{\node,j} = (\sum \globalp * \op{max}(flatten(diag(\Tilde{A_{\node:}})\hat{S}_j))) + H^1_{\node, \feature}, { where }\\
&\hat{S}_j = max(\op{max}(P\:diag(W_{:j})), 0)
\end{aligned}
\end{equation*}
where $\op{max}$ finds the maximum value for each row of the matrix. $min(\mathcal{S}_1)$ can be calculated similarly by using $\op{max}$ and $min(\cdot)$.



\paragraph{Concretization} The concretization of an abstract element $\absele = (\lowBound, \upBound) \in \abstractDomain_I$ is all matrices that within the interval of the abstract ele: $\gamma(\absele) = \{H \in \realDomain^{(\numNodes, \numFeatures)} | \forall \node < \numNodes, \feature < \numFeatures: \lowBound_{\node, \feature} \le H_{\node, \feature} \le \upBound_{\node, \feature}  \}$

\paragraph{Abstract operations}
Once the input abstraction is calculated, it can be processed through the abstract operations defined from the concrete operations. Let $(\lowBound, \upBound) \in \mathbb{A}$ be an element from the interval domain at some layer of the GCN, the abstract operations can be defined as follows \cite{liu2020abstract}:

\textbf{Fully Connected:} For $\linear{W, b}^{\#_I}$, the result can be calculated using \textit{interval arithmetic} as: 
\begin{equation*}
\label{equ:linear}
\begin{aligned}
    &\linear{W, b}^{\#_I}(\lowBound, \upBound) = (\lowBound', \upBound') \\
    &\lowBound'=\lowBound \times max(W, 0) + \upBound \times min(W, 0) + b \\
    &\upBound'=\upBound \times max(W, 0) + \lowBound \times min(W, 0) + b \\
\end{aligned}
\end{equation*}

\textbf{Graph Convolution:} Next, $\gc{}^{\#_I}$ only needs to perform the $\gc{}$ operation on the lower and upper bound.
\begin{equation*}
    \gc{}^{\#_{I}}(\lowBound, \upBound) = (\Tilde{A}\lowBound, \Tilde{A}\upBound) \\
\end{equation*}

\textbf{ReLU:} Similarly, the $\relu{}^{\#_{I}}$ function can also be calculated element-wise.
\begin{equation*}
    \relu{}^{\#_{I}}(\lowBound, \upBound) = (\relu{}(\lowBound),\relu{}(\upBound)) \\
\end{equation*}

\paragraph{Robustness certification}
Certifying the node robustness requires the difference $\delta_{\node, \lab, \lab'}$ between the scores $\op{cl}(\adj, \nodeF)_{\node, \lab}$ and \(\op{cl}(\adj, \nodeF)_{\node, \lab'}\). Given the output interval from the last layer \((\lowBound, \upBound) \in (\realDomain^{(\node, |\labelSet|)}, \realDomain^{(\node, |\labelSet|)})\), for a node $\node$, the lower bound of $\delta_{\node, \lab, \lab'}$ can be estimated by

\begin{equation*}
    min(\delta_{\node, \lab, \lab'}) = L_{\node, \lab} - U_{\node, \lab'}
\end{equation*}

Then, a certifier $\certifier_I$ can be derived with $r = \certifier_I(\graph, \perturba{X}, \classifier)$ such that for each node $\node$ in the graph, $r_{\node} = min([min(\delta_{\node, \lab_i, \lab'})| \lab' \ne \lab]) > 0$, where $\lab_i$ is the original label of the node $\node$.


\subsection{Proof for \autoref{thm:ai_soundness}}
\label{append:proof_ai_soundness}

\begin{apdthm}[Over-approximation of \absint{}]
\label{apdthm:ai_soundness}
Any abstract interpretation $(\abstractDomain{}, \concreteFunc{}, \operationSet{}^\#)$ defined for GCN as above over-approximates the behavior of the GCN classifier, i.e., \(\overline{\classifier}(A, \perturba{X}) = \{ \classifier(A, X') \mid X' \in \perturba{X} \} \subseteq \concreteFunc(\classifier^\# (\absele))\) for all abstract elements \(\absele \in \abstractDomain\) with \(\perturba{X} \subseteq \concreteFunc(\absele)\).
\end{apdthm}

Before we can proof \autoref{thm:ai_soundness}, we can observe that the set version of the operations $\overline{\operation}$ is monotone

\begin{lemma}[Monotonicity of $\overline{\operation}$]
    \label{lemma:monotone}
    The set function $\overline{\operation}(\concreteElement) = \{\operation(H)\mid H \in \concreteElement\}$ is monotonic: $\forall \concreteElement_1, \concreteElement_2 \in \concreteDomain: \concreteElement_1 \subseteq \concreteElement_2 \implies \overline{\operation}(\concreteElement_1) \subseteq \overline{\operation}(\concreteElement_2)$
\end{lemma}

\begin{proof}(\Cref{lemma:monotone})
Proof by contradiction: Let $\concreteElement_1, \concreteElement_2 \in \concreteDomain$ such that $\concreteElement_1 \subseteq \concreteElement_2$ but $\overline{\operation}(\concreteElement_1) \not \subseteq \overline{\operation}(\concreteElement_2)$. Then, there exist some matrix $H'$ such that $H' \in \overline{\operation}(\concreteElement_1)$ but $H' \not \in \overline{\operation}(\concreteElement_2)$. By the definition of $\overline{\operation}$, there must exist some $H$ such that $H' = \operation(H)$ and $H \in \concreteElement_1$ but $H \not \in \concreteElement_2$. This contradicts to $\concreteElement_1 \subseteq \concreteElement_2$
\end{proof}

Then, we can proof \autoref{thm:ai_soundness}. Since the adjacency matrix $\adj$ stays the same, it can be treated as a component of the operation $\operation$ so we omit $\adj$ for $\classifier$ in the proof
\begin{proof}(\autoref{thm:ai_soundness}) 
Let \(\classifier = \operation_0 \circ \operation_1 \circ \dots \circ \operation_z \) be \emph{any} GCN with \(z\) operations, \(\overline{\classifier} = \overline{\operation}_0 \circ \overline{\operation}_1 \circ \dots \circ \overline{\operation}_z \) be the corresponding set version of the GCN and \(\classifier^\# = \operation^\#_0 \circ \operation^\#_1 \circ \dots \circ \operation^\#_z \) be the GCN with corresponding abstract operations. 
Let \(\concreteElement_0 = \perturba{X}\) be a perturbation region and $\absele_0$ be an abstraction of \(\concreteElement_0\) such that \(\concreteElement_0 \subseteq \concreteFunc(\absele_0)\). Then \autoref{thm:ai_soundness} can be proved using induction.

\textbf{Base Case:} \(\concreteElement_0 \subseteq \concreteFunc(\absele_0)\) is defined for the input

\textbf{Hypothesis:} let $\concreteElement_l = (\overline{\operation}_0 \circ \overline{\operation}_1 \circ \dots \circ \overline{\operation}_{l-1})(\concreteElement_0)$ and $\absele_l = (\operation^\#_0 \circ \operation^\#_1 \circ \dots \circ \operation^\#_{l-1})(\absele_0)$, then $\concreteElement_l \subseteq \gamma(\absele_l)$

\textbf{Inductive Step:} Since $\concreteElement_l \subseteq \concreteFunc(\absele_l)$, using the monotonicity of $\overline{\operation}_l$ (\Cref{lemma:monotone}), we can derive $\overline{\operation}_l(\concreteElement_l) \subseteq \overline{\operation}_l(\concreteFunc(\absele_l))$. By the definition of an abstract operator, we have \( \forall \absele \in \abstractDomain: \overline{\operation_l}(\concreteFunc(\absele)) \subseteq \concreteFunc(\operation_l^\# (\absele))\). Then we can get $\concreteElement_{l+1}=\overline{\operation}_l(\concreteElement_l) \subseteq \overline{\operation}_l(\concreteFunc(\absele_l)) \subseteq \concreteFunc(\operation_l^\# (\absele_l))=\concreteFunc(\absele_{l+1})$, thus $\concreteElement_{l+1} \subseteq \concreteFunc(\absele_{l+1})$

Overall, we can conclude that \(\overline{\classifier}(A, \perturba{X}) \subseteq \concreteFunc(\classifier^\# (\absele))\) for all abstract elements \(\absele \in \abstractDomain\) with \(\perturba{X} \subseteq \concreteFunc(\absele)\), and this proves \autoref{thm:ai_soundness}.
\end{proof}


\subsection{Proof for \autoref{thm:relu_bounding}}
\label{append:proof_relu_bounding}

\begin{apdthm}
    \label{apdthm:relu_bounding}
    Let $\variable{}$ be a variable in interval $[\low{}, \high{}]$ such that $\low{} \le 0 \le \high{}$. The \textbf{minimum area bounding} of $\relu{}(\variable{})$ with one upper and one lower bound can be defined as:\\
    (1) upper bound: $\variable{}'=\frac{\high{}}{\high{}-\low{}}\variable{} - \frac{\high{} \cdot \low{}}{\high{} - \low{}}$\\
    (2) lower bound: $\variable{}'= \begin{cases}
            \variable{} & \text{if } |\high{}| \ge |\low{}| \\
            0 & \text{otherwise.}
        \end{cases}$
\end{apdthm}

\begin{proof}
 Let $\variable$ be any variable with lower bound and upper bound $l < 0 < u$. Let $\variable_{\ge}=\frac{u}{u-l}\variable - \frac{ul}{u-l}$ be the upper bound expression after $ReLU$, the lower bound $\variable_\le$ can be expressed as $\variable_\le=\lambda \variable, \lambda \in [0, 1]$. Then the resulting shape is bounded by 4 lines: $\variable'=\frac{u}{u+l}\variable - \frac{ul}{u-l}, \variable'=l, \variable'=u, \variable'=\lambda \variable$ forming a trapezoid. Then, the bounding area $Area$ can be calculated using the trapezoid area formula as $Area=\frac{1}{2}(-\lambda\cdot l+u-\lambda u)(u-l)$. This formula can be simplified as $Area=\frac{1}{2}(-(u^2-l^2)\lambda + u(u-l))$. Since $l, u$ are constant, the $\lambda$ minimizing $-(u^2-l^2)\lambda$ will also minimize $Area$. When $u^2 \ge l^2$, $-(u^2-l^2)\lambda \le 0$ and is minimized by $\lambda = 1$. When When $u^2 < l^2$, $-(u^2-l^2)\lambda \ge 0$ and is minimized by $\lambda=0$. Thus, when $|u| \ge |l|$, $x_{\le}=\variable$ will minimize the area and when $|u| < |l|$, $\variable_{\le}=0$ will minimize the area.
\end{proof}

\subsection{Proof for \autoref{thm:poly_sound}}
\label{append:proof_poly_sound}

\begin{apdthm} (Soundness of robustness certification)
\label{apdthm:poly_sound}
$\certifier{}_{\le}: (\adj{}, \perturba{\nodeF{}}, \classifier{}) \to \realDomain{}^\numNodes{}$ is sound; hence it produces a lower bound for the robustness of $\classifier$ over $\graph$.
\end{apdthm}

First, we show that the input abstraction over-approximates the input perturbation space.
\begin{lemma}
    \label{lemma:input_abstraction}
    The input abstraction over-approximates the perturbation space: $\perturba{\adj} \subseteq \gamma(\absele_0)$.
\end{lemma}

\begin{proof} (\Cref{lemma:input_abstraction})
    \(\absele_0 = ((I^{\numFeatures{}_0}, \mathbf{0}^{\numFeatures{}_0}, I^{\numFeatures{}_0}, \mathbf{0}^{\numFeatures{}_0}), \nodeVar_i)_{i = 1}^n \in \abstractDomain\). For any perturbation space $\perturba{\adj{}}$, the concretization of $\absele_0$ is the set of matrices satisfy the abstract domain and the perturbation space. Since $\absele_0$ bounds each input feature variable by itself, we have: \(\concreteFunc(\absele_0) = \{ H \in \mathbb{R}^{(n,m)} \mid H \vDash \absele_0 \cup \perturba{\adj}\} = \perturba{\adj{}}\). Thus, $\perturba{\adj} \subseteq \gamma(\absele_0)$.
\end{proof}

Next, we show that each abstract operation overapproximates the original operation. Since the operation for each node is done independently, we show the overapproximation for the feature values of each node. 

Let $\absele'$ be the abstract element after any operation $\operation^{\#}$: $\absele' = \operation^{\#}(\absele)$. To prove the overapproximation, it is suffice to prove that $\forall H \in \gamma(\absele): \operation(H) \in \gamma(\absele')$.

\begin{lemma}
    \label{lemma:poly_linear}
    The $\linear{\weight, \bias}^{\#_P}$ abstract operation overapproximate the behavior of  $\linear{\weight, \bias}$
\end{lemma}

\begin{proof} (\Cref{lemma:poly_linear})
    For any abstract element $\absele$, let $H \in \gamma(\absele)$ be \emph{any} concretization of the element for \emph{any} perturbation region $\perturba{\adj{}}$. Then for the latent feature variable $\texttt{h}_{i, j}$, we have $\sum_{\texttt{x}_k \in \map{}_i}\! \bigl((\polyLowExp{\node{}})_{j,k} \cdot \texttt{x}_j\bigr) + (\polyLowConst{\node{}})_j \le \texttt{h}_{i,j} \le \sum_{\texttt{x}_k \in \map{}_i}\! \bigl((\polyUpExp{\node{}})_{j,k} \cdot \texttt{x}_k\bigr) + (\polyUpConst{\node{}})_j$. Let $(\texttt{h}_{i, j})_{\le}$ and $(\texttt{h}_{i, j})_{\ge}$ be the minimum and maximum value of $\texttt{h}_{i, j}$. We also have $(\texttt{h}_{i, j})_{\le} \le H_{i, j} \le (\texttt{h}_{i, j})_{\ge}$. Let $H'_{i, k} = \sum_{j}{(H_{i, j} \cdot \weight_{j, k} + \bias_k)}$ be the value of the latent feature $k$ after $\linear{\weight, \bias}^{\#_P}$. If $(\texttt{h}_{i, j})_{\le} \le \texttt{h}_{i, j} \le (\texttt{h}_{i, j})_{\ge}$, we will have $-(\texttt{h}_{i, j})_{\ge} \le -\texttt{h}_{i, j} \le -(\texttt{h}_{i, j})_{\le}$. Let $W^+ = max(\weight, 0)$ and $W^- = min(\weight, 0)$ Then we have
    \begin{align*}
        &\sum_{j}{((\texttt{h}_{i, j})_{\le} \cdot W^+_{i, j} + (\texttt{h}_{i, j})_{\ge} \cdot W^-_{i, j}) + \bias_k} \le H'_{i, k} \\
        &\sum_{j}{((\texttt{h}_{i, j})_{\ge} \cdot W^+_{i, j} + (\texttt{h}_{i, j})_{\le} \cdot W^-_{i, j}) + \bias_k} \ge H'_{i, k}
    \end{align*}
    This is the definition of $\linear{\weight, \bias}^{\#_P}$. And this proves the overapproximation of $\linear{\weight, \bias}^{\#_P}$
    
\end{proof}

\begin{lemma}
    \label{lemma:poly_gc}
    The $\gc^{\#_P}$ abstract operation overapproximate the behavior of  $\gc$
\end{lemma}
\begin{proof} (\Cref{lemma:poly_gc})
    For any abstract element $\absele$, let $H \in \gamma(\absele)$ be \emph{any} concretization of the element for \emph{any} perturbation region $\perturba{\adj{}}$. Then for the latent feature variable $\texttt{h}_{i, j}$, we have $\sum_{\texttt{x}_k \in \map{}_i}\! \bigl((\polyLowExp{\node{}})_{j,k} \cdot \texttt{x}_j\bigr) + (\polyLowConst{\node{}})_j \le \texttt{h}_{i,j} \le \sum_{\texttt{x}_k \in \map{}_i}\! \bigl((\polyUpExp{\node{}})_{j,k} \cdot \texttt{x}_k\bigr) + (\polyUpConst{\node{}})_j$.
    Let $(\texttt{h}_{i, j})_{\le}$ and $(\texttt{h}_{i, j})_{\ge}$ be the minimum and maximum value of $\texttt{h}_{i, j}$ governed by $\perturba{\adj}$. We have $(\texttt{h}_{i, j})_{\le} \le \texttt{h}_{i,j} \le (\texttt{h}_{i, j})_{\ge}$. Notice that each entry of $\Tilde{A}$ is non-negative. Let $H' = \gc({H})$, $\sum_{v\in \neighbors{i}{1}}{\tilde{A}_{k, j}(\texttt{h}_{i, j})_{\le}}\le\sum_{v\in \neighbors{i}{1}}{\tilde{A}_{k, j}\cdot\texttt{h}_{k, j}}\le\sum_{v\in \neighbors{i}{1}}{\tilde{A}_{k, j}(\texttt{h}_{i, j})_{\ge}}$. At the same time, let $\texttt{h}'_{i,j}$ be the latent feature variable after applying the $\gc^{\#_P}$ operation, we have $\sum_{v\in \neighbors{i}{1}}{\tilde{A}_{k, j}(\texttt{h}_{i, j})_{\le}}\le \texttt{h}'_{i, j}\le\sum_{v\in \neighbors{i}{1}}{\tilde{A}_{k, j}(\texttt{h}_{i, j})_{\ge}}$. Thus, $(\texttt{h}'_{i, j})_{\le} \le \texttt{h}'_{i,j} \le (\texttt{h}'_{i, j})_{\ge}$ for any node $\node$ and feature $\feature$. And this proves the overapproximation of $\gc^{\#_P}$
\end{proof}

\begin{lemma}
    \label{lemma:poly_relu}
    The $\relu^{\#_P}$ abstract operation overapproximate the behavior $\relu$
\end{lemma}

\begin{proof} (\Cref{lemma:poly_relu})
    Since $\relu^{\#_P}$ applies to each variable independently, we show the proof that for any latent variable $\texttt{h}_{i, j}$ with numerical bound $lo_{i, j}$ and $up_{i, j}$. Let $H \in \gamma(\absele)$, we have $lo_{i, j} \le H_{i, j} \le up_{i, j}$. Let $\textit{h}'_{i,j}$ be the latent of the abstract element after $\relu^{\#_P}$. We show the prove with three cases:

    \textbf{Case 1:} if $lo_{i, j} \ge 0$, we have $H'_{i, j} = \relu(H_{i, j}) = H_{i, j}$ and $\textit{h}'_{i,j} = \textit{h}_{i,j}$. Thus $lo'_{i, j} \le H'_{i, j} \le up'_{i, j}$.
    \textbf{Case 2:} if $up_{i, j} < 0$, we have $H'_{i, j} = \relu(H_{i, j}) = 0$ and $\textit{h}'_{i,j} = 0$, thus $lo'_{i, j} \le H'_{i, j} \le up'_{i, j}$.
    \textbf{Case 3:} if $lo_{i, j} < 0 < up_{i, j}$. We have $0 \le H'_{i, j} \le up_{i, j}$. For $\textit{h}'_{i,j}$, we have one upper bound and two potential lower bounds. For the upper bound, we have $\textit{h}'_{i,j} \le \frac{up_{i, j}}{up_{i, j}-lo_{i, j}}\textit{h}_{i,j} - \frac{up_{i, j} \cdot lo_{i, j}}{up_{i, j}-lo_{i, j}} \le up_{i, j} = up'_{i, j}$. Thus, $H'_{i, j} \le up'_{i, j}$. For the first lower bound $lo'_{i, j} = 0 \ge \textit{h}'_{i,j}$ and for the second lower bound $lo'_{i, j} = lo_{i, j} \le \textit{h}_{i,j} \ge \textit{h}'_{i,j}$. In both cases, we have $lo'_{i, j} \le H'_{i, j}$.

    Overall, for any $lo_{i, j} \le H_{i, j} \le up_{i, j}$, we have $lo'_{i, j} \le H'_{i, j} \le up'_{i, j}$. And this proves the overapproximation of $\relu^{\#_P}$
\end{proof}

Then, by the overapproximation of abstract interpretation, the polyhedra interpretation with the defined input abstraction and abstract operators will provide a sound overapproximation of the GCN's behavior.

\begin{lemma}[Overapproximation of polyhedra \absint]
    \label{lemma:poly_overapproximate}
    Let $(\abstractDomain_P, \gamma, \operationSet^{\#_P})$ be a polyhedron abstract interpretation with $\operationSet^{\#_P} = \{\linear{\weight, \bias}^{\#_P}, \gc^{\#_P}, \relu^{\#_P}\}$ and the input abstraction, then it overapproximate the behaviour of the GCN classifier.
\end{lemma}

\begin{proof} (\Cref{lemma:poly_overapproximate})
    By \Cref{lemma:poly_linear,lemma:poly_gc,lemma:poly_relu}, all three operations \(\overline{\operation}(\concreteFunc(\absele)) \subseteq \concreteFunc(\operation^\# (\absele))\) for all \(\absele \in \abstractDomain\). Let $\perturba{X} \in \concreteDomain$ be any perturbation region, and $\absele \in \abstractDomain$ be the abstract element after applying the input abstraction to $\perturba{X}$. By \Cref{lemma:input_abstraction}, we have \(\perturba{X} \subseteq \concreteFunc(\absele)\). These satisfy the precondition of \autoref{thm:ai_soundness}, and thus, the polyhedra abstract interpretation provides an over approximation. 
\end{proof}

\begin{proof} (\autoref{thm:poly_sound})
    To prove \autoref{thm:poly_sound}, we need to prove for any node $\node$, the estimated minimum difference between the original label and any other label $\delta^*_{\node, \lab, \lab'}$ from the polyhedra abstract interpretation approach is the lower bound of the actual minimum difference for all possible outputs in the perturbation region. If this is true, then by the definition of $\nodeRC{\node{}}$ and $\nodeR{\node{}}$, we will always have $\nodeRC{\node{}} \le \nodeR{\node{}}$ for any node $\node$. Since both values are binary, this property can imply $\nodeRC{\node{}} \implies \nodeR{\node{}}$, which is the soundness property. 

    Given $\graph = (\adj, \nodeF)$ as the input graph with perturbation region $\perturba{\adj}$, let $\concreteElement_0 \in \concreteDomain$ be the set of all possible input feature matrix in the perturbation region. Let $\absele_0 \in \abstractDomain_P$ be abstraction of $\concreteElement_0$ using he input abstraction. From \Cref{lemma:input_abstraction}, we have $\concreteElement_0 \subseteq \gamma(\absele_0)$. Furthermore, Given a GCN $\classifier$, its set version $\overline{\classifier}$, let $\classifier^{\#_P}$ be the abstract version of $\classifier$ using polyhera abstract operations. Let $\concreteElement_\numLayers = \overline{\classifier}(A, \concreteElement_0)$ and $\absele_\numLayers = \classifier^{\#_P}(A, \absele)$. From \Cref{lemma:poly_overapproximate}, we have $\concreteElement_\numLayers \subseteq \gamma(\absele_\numLayers)$

    For any node $\node$ with original label $\lab$ and another label $\lab' \ne \lab$, let $\concreteElement_{\numLayers, \node} = \{H_i|H\in\concreteElement_\numLayers\}$ be the set of all output score vectors of $\node$. Then the set of all possible differences between output score values in $\concreteElement_{\numLayers, \node}$ is $\concreteElement'_{\numLayers, \node, \lab'} = \overline{\linear{\Delta^{\lab, \lab'}, 0}}(\concreteElement_{\numLayers, \node})$. Let $(\coef{}_i, \map_i) \in \absele_\numLayers$ be the part of abstract element for node $\node$.  And the abstraction of all possible differences is given by $\coef'_{\node, \lab'} = \linear{\Delta^{\lab, \lab'}, 0}^{\#_P}(\coef_{i})$. Now, since we focus on node $\node$, let us define a a new abstract element for node $\node$ only: $\absele_i = (\coef'_{\node, \lab'}, \map'_i)$. Let $\hat{\delta}^*_{\node, \lab, \lab'} = min(\concreteElement'_{\numLayers, \node, \lab'})$ be the actual minimum difference and $\delta^*_{\node, \lab, \lab'}$ be the solution calculated using the steps in \autoref{sec:rbst-cert}. Notice that $\delta^*_{\node, \lab, \lab'} = min(\gamma(\absele'))$ since it is the solution to the minimization problem for the concretization of an abstract element. Since $\concreteElement'_{\numLayers, \node, \lab'}\subseteq \gamma(\absele')$, we can derive $\delta^*_{\node, \lab, \lab'} \le \hat{\delta}^*_{\node, \lab, \lab'}$

    Then, for any node $\node$, let $\hat{\delta}^*_{\node, \lab} = [\hat{\delta}^*_{\node, \lab, \lab'}]_{\lab' \ne \lab}$ and $\delta^*_{\node, \lab} = [\delta^*_{\node, \lab, \lab'}]_{\lab' \ne \lab}$, we can have $min(\delta^*_{\node, \lab}) \le min(\hat{\delta}^*_{\node, \lab})$. And this proves the soundness of the polyhedra abstract interpretation. 

\end{proof}

\subsection{Proof for \autoref{thm:poly_compelte}}
\label{append:proof_poly_complete}

\begin{apdthm} (Completeness of counterexamples)
\label{apdthm:poly_compelte}
$\certifier{}_{\ge}: (\adj{}, \perturba{\nodeF{}}, \classifier{}) \to \realDomain{}^\numNodes{}$ is complete and provides an upper bound for graph-level robustness.
\end{apdthm}

\begin{proof}
    Assume that $\certifier_\ge$ is not complete, let $\node$ be a node such that $\node$ is robust ($\nodeR{\node} = 1$) and  $\nodeRCV{\node}{\certifier_{\ge}}=0$, thus $r^{*\prime}_i=0$. From \autoref{equ:node_robustness_min}, node $\node$ is robust only if:
    \begin{equation*} 
    \forall \nodeF{}' \in \perturba{\nodeF{}}\colon \argmax{\classifier(\adj{}, \nodeF{})_\node{}} = \argmax{\classifier(\adj{}, \nodeF{}')_\node{}}.\\
    \end{equation*}
    However, by the definition of $\certifier_\ge$, a counter example is identified for $\node$ such that $argmax(\classifier(\adj,\nodeF)_\node) \ne argmax(\classifier(\adj,\nodeF^{\node\lab'})_\node)$. This contradicts with node $\node$ being robust. Thus, using proof by contradiction, we can conclude that $\certifier_\ge$ is complete. And a complete certifier produces an upper bound on the robustness. 
\end{proof}




\subsection{Implementation details}
The polyhedra abstract operations can be implemented using matrix operations since the transformation of each node feature is independent with a straightforward modification: the input abstraction does not give an identity matrix. Instead, it yields a matrix of shape $\numFeatures_0 \times (\numNodes \times \numFeatures_0)$ with the entry corresponding to feature $j$ of node $i$ being 1. This matrix may be memory intensive if $\numNodes$ is large. Nevertheless, we can mitigate this by (1) using a sparse matrix and (2) using batched computation such that $\numNodes$ is not the entire graph but the subgraph of $\layer$-hop neighbors from the nodes in the batch. 
With matrix operations, the certification can be run efficiently on GPUs. Furthermore, the certification process is differentiable such that the certifier can be combined with computation frameworks such as PyTorch or TensorFlow for robust training.

\subsection{Training configuration}
\label{sec:training_config}
We implement both \emph{\Ours{}} and \emph{\Interval{}} approaches in Pytorch. We randomly split the nodes into training and testing sets for five random nodes and trained with the same random nodes. 
During training, we use a learning rate of 0.005 and weight decay of 0.00001 with an Adam optimizer. We use a batch size of 8 and accumulate a gradient for every five batches, making a 40-node batch per gradient update. We use node sampling during training. We set the training step to 1500 for experiments with labeled nodes and 15000 for experiments with both labeled and unlabeled nodes for all datasets. Finally, there are some nodes in Pubmed causing the \Dual{} approach to crash; we exclude these nodes during training for both approaches. For the hinge loss, we use a margin of $log\frac{90}{10}$ for labelled nodes and $log\frac{60}{40}$ for unlabelled nodes, same as previous work \cite{zugner2019certifiable}. 


\end{document}